\documentclass{article}
\usepackage{xcolor}
\usepackage{float}

\usepackage[preprint]{corl_2026}

\usepackage{graphicx}
\usepackage{booktabs}
\usepackage{tabularx}
\usepackage{needspace}
\usepackage{amsmath}
\usepackage{amssymb}
\usepackage{amsthm}
\usepackage{algorithm}
\usepackage{algpseudocode}
\algnewcommand{\IfThen}[2]{\State \algorithmicif\ #1\ \algorithmicthen\ #2}
\newtheorem{proposition}{Proposition}
\newtheorem{lemma}{Lemma}
\usepackage{tikz}
\usetikzlibrary{positioning,arrows.meta,fit,calc}

\newcommand{\sense}{\textup{\textsc{Sense}}}
\newcommand{\plan}{\textup{\textsc{Plan}}}
\newcommand{\act}{\textup{\textsc{Act}}}
\newcommand{\spa}{SPA}
\newcommand{\rsi}{RSI}

\newcommand{\dworld}{\mbox{D-World}}
\title{D-World: An Explicit World-Language-Action Model\\with Recursive Self-Improvement under Dual Control}
\hypersetup{pdftitle={D-World: An Explicit World-Language-Action Model with Recursive Self-Improvement under Dual Control}, pdfauthor={Xuan Wen and Baichao Wang}, pdfsubject={}}

\author{
  Xuan Wen\\
  Geek+\\
  \texttt{wenxuan23@mails.ucas.ac.cn}
  \And
  Baichao Wang\\
  Geek+\\
  \texttt{baichao.wang@geekplus.com}
}

\begin{document}
\raggedbottom
\brokenpenalty=10000
\maketitle

\begin{abstract}
Frontier large language models (LLMs) now reason about space well enough to program robots, yet the world models behind robot policies are implicit, held in network weights or predicted video, and cannot be inspected or corrected directly.
We present \dworld{}, an explicit world-language-action model built on the Sense, Plan and Act decomposition.
A perception agent writes the scene as an executable digital twin using generated vision and geometry code.
A planning agent searches for actions in simulation and, after each executed action chunk, revises the twin's object states and named physical parameters against what the robot measured.
A parameter revision is accepted only if replaying the latest chunk does not increase any measured residual beyond its tolerance, and every chunk is checked in the updated twin before it runs. Structural mismatches return to perception for reconstruction.
We call this measurement-grounded cycle of correction and revision recursive self-improvement (RSI), in which the LLM stays frozen while the world it acts in improves.
We model the choice of actions that both advance the task and expose model error as a dual control problem and show that the cost of incorporating evidence into the twin sets the threshold at which probing pays, so cheap local revision counteracts the turn-off of cautious controllers.
Each episode runs a real-to-sim-to-real loop, and its recorded twins can be refitted and re-rendered offline into higher-fidelity environments for post-training.
Because perception and planning communicate only through the twin and language, either can be replaced by a learned model, and since inference dominates the loop, the loop speeds up with faster models.
On a dual-arm robot, towel and toaster tasks show local revisions, such as one reducing the bread's rendered-depth residual from 22.4 to 1.5\,mm, that let execution continue without reconstruction.
\end{abstract}


\begin{figure}[t]
\centering
\includegraphics[width=\linewidth]{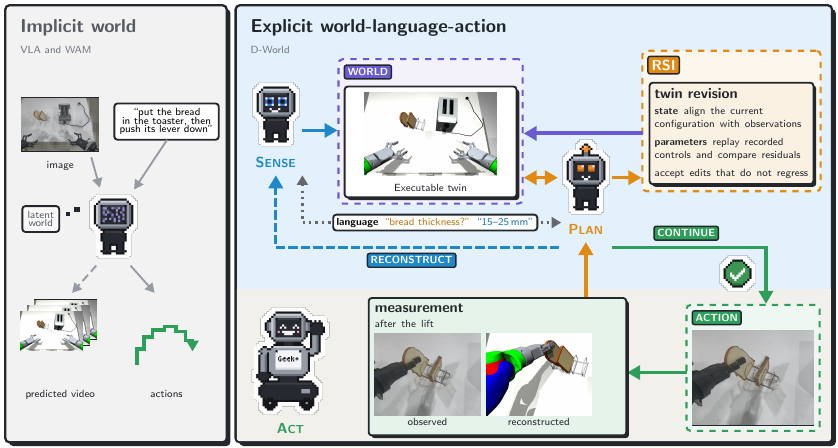}
\caption{\textbf{\dworld{}, an explicit world-language-action model.} Implicit models (left) keep the world in network weights or predicted video. \dworld{} (right) writes it as an executable twin. \sense{} constructs the twin, \plan{} proposes trajectories and twin revisions, and a gate checks each action chunk before execution. Measurements of the real robot drive state correction and replay-checked parameter revision (\rsi{}), and the dashed path requests reconstruction. In a bread-grasp episode, the twin shows both grippers open, ACTION shows the robot holding a lifted bread slice, and the measurement pair compares this post-lift observation with the reconstructed twin. In the language channel, \plan{} asks and \sense{} answers.}
\label{fig:overview}
\end{figure}

\section{Introduction}
\label{sec:intro}

Frontier multimodal LLMs now reason about scene geometry and spatial relations well enough to write robot programs that compose perception and control primitives~\citep{liang2023cap,fu2026capx}, or to act on a robot turn by turn~\citep{zhang2026astrapolicy,galbot2026astra}.
This makes agentic robotics practical, but it leaves open where the agent's predictions about the physical consequences of its actions come from.
Precise grasps, contact-rich motion and bimanual coordination remain difficult when an error becomes visible only after execution~\citep{chen2026robodojo,galbot2026astra}.

World-action and world-language-action models address this by predicting the future together with the action~\citep{ye2026dreamzero,lu2026wamsurvey,yang2026wla}.
Their world, however, is implicit and resides in network weights or predicted video.
When such a prediction disagrees with the scene, the disagreement cannot be traced to a particular quantity, and correcting it requires new data and training.
An explicit world has the complementary property.
A coding agent can write a physical scene as a program that a simulator executes~\citep{wang2026codeasworld}, so that every prediction follows from named geometry and physical parameters.
A reconstruction written in one pass still disagrees with the scene~\citep{tang2026video2world,tang2026bvb}, yet each disagreement now points to fields that can be edited and tested.

Building on this property, we call a system an \emph{explicit world-language-action model} when its world is an executable program that can be inspected and revised, language is the medium in which agents write, revise and question that world, and every action is checked in that world before execution. Unlike world-language-action models that unify the three in a single network~\citep{yang2026wla}, such a system keeps the world explicit.
\dworld{} is an explicit world-language-action model organized around the classical Sense, Plan and Act (\spa{}) decomposition~\citep{nilsson1984shakey} (Fig.~\ref{fig:overview}).
Its world is an executable digital twin, hereafter the twin, which a perception module, \sense{}, writes from the robot's cameras.
A planning module, \plan{}, searches for trajectories by simulation, commits a short prefix, or \emph{chunk}, for execution, and compares the measured response with its prediction.
Classical \spa{} was criticized because its model drifts from the world~\citep{brooks1986robust,brooks1991intelligence}. \dworld{} instead treats this drift as the signal from which the model is improved.

After each chunk, \plan{} explains the measured discrepancy by correcting object states or revising named parameters such as dimensions and friction.
A harness accepts a parameter revision only if replaying the latest chunk in the revised twin does not increase any measured residual beyond its tolerance, and it then gates the next chunk in the updated twin.
Structural mismatches return to \sense{} for reconstruction.
We call this measurement-grounded cycle of correction and revision \emph{recursive self-improvement} (\rsi{}), with reconstruction as its structural fallback. The LLM remains frozen, while the world in which it plans improves.
Each episode thereby runs a real-to-sim-to-real loop~\citep{torne2024rialto} at the granularity of action chunks, and the recorded twins can afterwards be refitted and re-rendered offline into higher-fidelity environments for post-training.

Because the twin improves through action, the choice of action also determines what can be learned.
Dual control formalizes this coupling between control and estimation~\citep{feldbaum1960dual,barshalom1974dual}.
We treat the twin, together with the set of parameter values that remain consistent with the measurements, as the belief of the loop. Each chunk both advances the task and changes this belief.
A controller that only avoids uncertain contact never reduces the uncertainty that motivated its caution, which is the turn-off phenomenon of adaptive control~\citep{astrom1971problems,wittenmark1995dual}.
We show that probing pays below a success-probability threshold that falls with the cost of incorporating the probe's outcome into the model, so local revision, which absorbs that outcome in one replay-checked edit instead of a full reconstruction, widens the range of beliefs in which the planner should probe rather than wait.

Finally, \sense{} and \plan{} communicate only through the explicit twin with its uncertainty ranges, its revisions, the measured evidence, and notes and questions in language.
Either module can therefore be replaced by a learned reconstruction model or policy without changing the checks, and since model inference dominates the loop's running time, the loop becomes faster and more accurate as the underlying models improve.

\Needspace{5\baselineskip}
Our contributions are as follows.
\begin{itemize}
\setlength\itemsep{0pt}
\item \textbf{An explicit world-language-action model.} \dworld{} represents the world as an executable twin that a frozen LLM agent writes from robot observations, and it plans, checks and acts in this twin, with perception and planning decoupled through it.
\item \textbf{Recursive self-improvement within an episode.} Measured discrepancies drive state corrections and replay-checked parameter revisions, with reconstruction reserved for structural change. A dual-control formulation characterizes informative chunks by the measurement tolerance and shows that inexpensive local revision raises the success probability below which probing pays, counteracting turn-off.
\item \textbf{A harness for physical agents, with real-robot evidence.} The harness manages the agents' contexts and tools, verifies every revision and chunk, and keeps the execution records. On a dual-arm robot, towel and toaster tasks show revision, reconstruction and continued execution, and the episode records support offline refitting of higher-fidelity twins.
\end{itemize}

\section{Related Work}
\label{sec:related}

\paragraph{Implicit and explicit world models.}
Vision-language-action models map observations and language directly to actions~\citep{pi2025pi05,nvidia2025groot}.
World-action models additionally predict future observations~\citep{ye2026dreamzero,yuan2026fastwam,lu2026wamsurvey}, and world-language-action models unify world modeling, language reasoning and action synthesis in a single network~\citep{yang2026wla}.
In these models the world is implicit.
Explicit alternatives represent dynamics as programs refined from interaction~\citep{tang2024worldcoder,dainese2024gifmcts}, build object-level world models for zero-shot manipulation~\citep{li2026explicitwm}, or let coding agents construct physical scenes in executable form~\citep{li2026rcwm,wang2026codeasworld,chen2026codeworldmodel}.
VisualPatchWorld fits a dynamics program from probes and plans with it~\citep{bai2026visualpatchworld}.
\dworld{} pairs an executable world with language-model agents and closes the loop through measurements of the real robot.

\paragraph{Real-to-sim-to-real.}
Reconstructed scenes support rollout-based planning~\citep{ning2025pwtf,liu2026simpact}, twin synchronization for a learned policy~\citep{abouchakra2025realissim} and policy training~\citep{torne2024rialto,wang2026phys2real}.
LLM-based reconstruction and policy-development systems use feedback between real trials~\citep{hu2026simex}, offline replays~\citep{chen2026agenticreal2sim} or workspace observations~\citep{li2026agenticrsr}.
\dworld{} runs the real-to-sim-to-real cycle within an episode. Each executed chunk returns measurements, a replay check accepts revisions, and the updated twin plans and gates the next chunk.

\paragraph{Robot agents and harnesses.}
LLMs control robots through programs~\citep{liang2023cap,huang2023voxposer,huang2024rekep}, coding-agent workflows~\citep{fu2026capx,chu2026openrua} or turn-by-turn interaction~\citep{tsui2026faea,zhang2026astrapolicy}, and their perception may compose pretrained modules~\citep{suris2023vipergpt} or use agent-written measurement code~\citep{talwar2026localcoding,chu2026openrua}.
The harness of such an agent, comprising its prompts, tools, context management and control rules, largely determines how well it performs~\citep{han2026vista,chu2026openrua,wei2026harnessevolution}. VISTA supports a general-purpose multimodal model through visual memory and model-directed inspection~\citep{han2026vista}, SafeHarness adds verified obstacle-aware route planning~\citep{xu2026safeharness}, and RegenHarness gates task-state commits on evidence and revises its own configuration under regression checks~\citep{wang2026regenharness}.
Since an LLM's own judgment is an unreliable verifier~\citep{kambhampati2024llmmodulo}, robot systems check kinematics and collisions~\citep{hong2026bridging} or monitor execution in images~\citep{zhou2025codeasmonitor}.
Our harness centers these functions on an editable physical twin. Its execution gate follows the model-based principle of predictive safety filters~\citep{wabersich2018psf}, applied to planner-specified predicates in the nominal twin.

\paragraph{Self-improvement, identification and dual control.}
Agents accumulate skills~\citep{wang2023voyager}, revise their work from observations and failures~\citep{yao2023react,shinn2023reflexion}, and improve harnesses or policies across tasks~\citep{lee2026rhsi,xia2026rrsi,lang2026embodirsi,song2026embodiedrsi}, including evidence-gated revision of a robot harness~\citep{wang2026regenharness}. Such improvement is bounded by the quality of its feedback~\citep{chen2026rsisurvey,wang2026whatstops}.
In \dworld{}, the object being improved is the world model rather than the harness or the policy, and the feedback is physical measurement.
Parameter revision is therefore a form of system identification, related to LLM-guided fitting through physics engines~\citep{cherian2026llmphy,wang2026phys2real} and refinement from prediction errors~\citep{zhang2026worldevolver}.
EMPIRIC extends a physics engine with residual code and selects experiments by their information~\citep{liang2026empiric}, whereas \dworld{} revises the parameters of a supported scene representation and leaves structural change to reconstruction.
Dual control accounts for how actions affect both the system and future information~\citep{feldbaum1960dual,barshalom1974dual,barshalom1976caution,wittenmark1995dual}. Practical approaches use sampling~\citep{mesbah2018dual}, set-membership models~\citep{lorenzen2019rmpc} or shields~\citep{hu2024dualshield}, and robotic probes can reveal intent or physical properties~\citep{sadigh2016info,memmel2024asid}.

\section{Problem Formulation}
\label{sec:problem}

\paragraph{Setting and representation.}
A robot acts on a language task $\ell$ in an unfamiliar scene, using a supplied kinematic model, controller and camera parameters.
An observation $o$ contains images, available depth and measured robot state.
An action specifies bounded end-effector and gripper targets for one or both arms.
The twin is $\hat{\mathcal T}_t=(\mathcal M_t,\hat\theta_t,\hat x_t)$, where $\mathcal M_t$ specifies bodies, geometry types and connectivity, $\hat\theta_t$ contains dimensions, masses, friction, contact parameters and fixed-body poses, and $\hat x_t$ contains movable-body poses and velocities and articulated-joint states.
\sense{} constructs this representation from observations.
\plan{} can revise supported parameter and state fields, whereas changes to body types or connectivity require reconstruction.
Planning and execution share the robot model and controller, although their simulated response can differ from the hardware.

\paragraph{Measurements and residuals.}
The index $t$ counts decision rounds, including rounds in which the gate refuses execution.
A proposed chunk $u_t$ is the prefix of a trajectory $\tau_t$ that \plan{} commits for execution.
If executed, it produces $y_t$, which comprises proprioception after each action, requested camera captures, and an observation at the chunk's end. A refused round produces no $y_t$.
The history $I_t$ contains the observations and executed-chunk records available before round $t$.
Each record $\mathcal D$ stores the starting twin and robot state, controls and measured response, and a static record contains an observation alone. The record of the chunk executed in round $t$ is $\mathcal D_t$.
For each measured signal $i$, the harness evaluates the residual entry
\begin{equation}
e_i(\mathcal T;\mathcal D)=d_i\big(y_i(\mathcal D),\hat y_i(\mathcal T;\mathcal D)\big),
\label{eq:error}
\end{equation}
where $\hat y_i$ is the prediction from the recorded controls and starting state, or a rendering for a static record, and $d_i$ is the corresponding discrepancy measure.
The entries cover end-effector paths, gripper closure, first resistance and rendered-depth residuals (App.~\ref{app:error}).
Their operational tolerances are derived from sensor specifications, observed noise or free-motion discrepancies.

\paragraph{Revision acceptance.}
Let $\bar{\mathcal D}_t$ be the latest record available at round $t$, namely the latest executed-chunk record since reconstruction, or the reconstruction observation if no chunk has executed since then.
The same record is used before and after a proposed revision $\Delta$, a collection of named field assignments rather than an additive parameter vector. A revision can contain state corrections, parameter changes, or both.
Write $e_i^{\mathrm{rep}}(\Delta;\bar{\mathcal D}_t)$ for the residual after applying its replayable changes to the recorded twin, including previously accepted changes applicable to that record.
The baseline $e_i^{\mathrm{rep}}(\varnothing;\bar{\mathcal D}_t)$ applies no new change.
An accepted revision must satisfy
\begin{equation}
e_i^{\mathrm{rep}}(\Delta;\bar{\mathcal D}_t)\le e_i^{\mathrm{rep}}(\varnothing;\bar{\mathcal D}_t)+\varepsilon_{i,t}
\quad\text{for all }i\in\mathcal E_t(\Delta),
\label{eq:accept}
\end{equation}
where $\mathcal E_t(\Delta)$ contains entries available in both comparisons with a defined tolerance, and $\varepsilon_{i,t}$ is fixed from the baseline comparison.
The set can depend on the candidate because rendered visibility can change, and a missing entry imposes no constraint.
Depth is evaluated at the replay's endpoint, or on the available depth for a static record.
Corrections to the pose of a body moved by the chunk, and to articulated-joint state, are excluded from replay, because they synchronize the current twin rather than explain the preceding transition. Pose corrections of undisturbed bodies are replayed.
The test thus permits changes within tolerance and enforces non-regression on the latest evidence. Effects on earlier chunks are reported to \plan{} without vetoing an update.

\paragraph{Execution gate.}
Each chunk carries expectations $\mathcal P_t$ expressed as predicates on its simulated outcome, including the movable bodies it affects.
The binary gate $\mathrm G(\hat{\mathcal T}_t^+,\delta_t)$ equals one exactly when the proposed motion respects the robot limits, the chunk's expectations hold and, if \plan{} declares the task done, its task predicates hold over the proposed trajectory. Here $\hat{\mathcal T}_t^+$ denotes the twin after revision selection and $\delta_t$ the full decision defined below.
The harness first selects the accepted revisions, possibly none, and then gates the plan in the resulting twin, so that a refused revision does not by itself refuse the action plan.
The gate certifies a chunk in the nominal simulation, which uses the current point estimates. Whether the scene follows the simulation is what the subsequent measurement tests, and whether the predicates capture the task is what \sense{}'s completion check tests.

\paragraph{Conditional consistency sets.}
To represent parameter uncertainty, let $\bar\Theta_{k,t}$ be the candidate interval for scalar parameter $\theta_k$, spanning \sense{}'s range and the values proposed since reconstruction.
For a candidate $v$, $e_{ij}^{(k)}(v;t)$ denotes residual $i$ on record $j$ when $\theta_k$ is varied, the other parameters are held at their current estimates, and replay starts from the recorded states.
The coordinate-wise conditional set is
\begin{equation}
\Theta_k(t)=\left\{v\in\bar\Theta_{k,t}:
 e_{ij}^{(k)}(v;t)\le\inf_{w\in\bar\Theta_{k,t}}e_{ij}^{(k)}(w;t)+\varepsilon_{ij,t}
 \quad\forall(i,j)\in\mathcal J_{k,t}\right\},
\label{eq:open}
\end{equation}
where $\mathcal J_{k,t}$ indexes comparisons defined across the candidate domain, and $\varepsilon_{ij,t}$ is the corresponding comparison tolerance.
The implementation evaluates this set on a finite grid, using motion since reconstruction and depth from the latest observation (App.~\ref{app:error}).
It reports intervals $\widetilde\Theta_k(t)$ whose product $\widetilde{\mathcal W}_t=\prod_k\widetilde\Theta_k(t)$ approximates the belief set $\mathcal W_t$ defined below and supplies parameter vectors for instantiating candidate twins.
This coordinate-wise diagnostic is related to set-membership estimation~\citep{milanese1991setmembership}. Unlike profile likelihood~\citep{raue2009identifiability}, it holds the other parameters at their estimates, so parameter coupling can make $\widetilde{\mathcal W}_t$ differ from the joint set $\mathcal W_t$ (App.~\ref{app:error}).
A reconstruction resets these intervals.

\paragraph{Decisions.}
The planning policy produces
\begin{equation}
\delta_t=\mu_\pi(\ell,I_t,\hat{\mathcal T}_t,h_t),
\label{eq:plan}
\end{equation}
where $\pi$ is the frozen LLM and $h_t$ contains the modules' notes and scripts.
The decision $\delta_t$ contains the trajectory $\tau_t$, chunk $u_t$, proposed revisions $\Delta_t$, predicates $\mathcal P_t$ and next step $\nu_t\in\{\textsc{continue},\textsc{reconstruct},\textsc{done}\}$, and it may also supply notes and future chunks.
Each evaluation of $\mu_\pi$ is a \emph{planning call}, within which \plan{} may issue many LLM requests and tool commands. A round makes one planning call and up to $R$ repair calls, or none when a listed chunk continues.
After revision selection and gating, the robot executes $u_t$, and the twin advances to the simulated endpoint when measured motion agrees, or by simulation along the measured motion otherwise.
If no call yields a passing chunk, the round executes nothing, and two consecutive such rounds terminate the episode.
A successful status requires \sense{}'s confirmation on a fresh observation. Other stopping conditions are resource limits and an operator stop.
Listed future chunks can bypass a new planning call when the continuation checks pass (App.~\ref{app:harness}), but each still passes the gate.

\paragraph{Cost.}
The time of a round decomposes as
\begin{equation}
T_t=T_{P,t}+T_{A,t}+T_{S,t}+T_{H,t},
\label{eq:cost}
\end{equation}
where the terms partition the round's critical path into planning (including repairs), execution, sensing, and other harness operations or waiting, with concurrent work counted once.
Sensing time is charged when initialization, reconstruction or completion checking occurs.

\paragraph{Belief and its update.}
The loop carries the belief $b_t=(\hat{\mathcal T}_t,\mathcal W_t,\rho_t)$, which consists of the updated twin, the parameter set $\mathcal W_t$, and a distribution $\rho_t$ supported on $\mathcal W_t$ from which \plan{} samples candidate twins.
An executed chunk updates it through
\begin{equation}
b_{t+1}=\beta(b_t,u_t,y_t),\qquad
\mathcal W_{t+1}=\big\{\theta\in\bar\Theta_{t+1}:\ d_i\big(y_{i,j},\hat y_i(\theta;\mathcal D_j)\big)\le\varepsilon_{i,j}\ \ \forall i\in\mathcal Y_j,\ \forall j\in\mathcal R_{t+1}\big\},
\label{eq:belief}
\end{equation}
where $\bar\Theta_{t+1}=\prod_k\bar\Theta_{k,t+1}$ is the candidate domain and $\mathcal R_{t+1}$ contains the records since reconstruction, the latest being that of $u_t$.
Here $\hat y_i(\theta;\mathcal D_j)$ is the prediction obtained by replaying the controls of record $\mathcal D_j$ from its starting state in the twin with the structure since reconstruction and parameters $\theta$, and $\mathcal Y_j$ contains the entries measured in that record.
The tolerances $\varepsilon_{i,j}$ are the operational tolerances, applied to the discrepancy between measurement and prediction rather than to the increase of a residual.
The point estimate $\hat{\mathcal T}_{t+1}$ is obtained by \rsi{} under Eq.~\eqref{eq:accept} and advanced to the end of the chunk. The distribution $\rho_{t+1}$ is $\rho_t$ conditioned on $\mathcal W_{t+1}$ when the candidate domain is unchanged and $\rho_t(\mathcal W_{t+1})>0$, and otherwise a distribution reinitialized on $\mathcal W_{t+1}$, which exists whenever $\mathcal W_{t+1}$ is nonempty.
When the candidate domain is unchanged, Eq.~\eqref{eq:belief} reduces to the recursion $\mathcal W_{t+1}=\{\theta\in\mathcal W_t:\ d_i(y_{i,t},\hat y_i(\theta;\mathcal D_t))\le\varepsilon_{i,t}\ \forall i\in\mathcal Y_t\}$, so the set can only shrink.
An empty $\mathcal W_{t+1}$ means that no parameter vector in the candidate domain explains the evidence within tolerance, so that only an enlarged domain or a reconstruction can restore consistency.
\rsi{} is the point-estimate part of $\beta$, and Eq.~\eqref{eq:belief} its set-valued part. The point estimate need not lie in $\mathcal W_{t+1}$, because acceptance requires non-regression rather than consistency.
The implementation approximates $\mathcal W_{t+1}$ by the coordinate-wise scan of Eq.~\eqref{eq:open}, which relaxes each tolerance by the best grid candidate's residual and returns the full range, marked uninformative, when no candidate passes (App.~\ref{app:error}).

\paragraph{Objective.}
The decision problem underlying \dworld{} is to complete the task in the least expected time,
\begin{equation}
\min_{\mu}\ \mathbb E^{\mu}\Big[\textstyle\sum_{t'=t}^{K}T_{t'}\ \Big|\ b_t\Big],
\label{eq:dual}
\end{equation}
where $K$ is the round in which completion is confirmed, and the expectation under policy $\mu$ is over physical parameters $\theta\sim\rho_t$ together with the resulting measurements.
Because $u_t$ influences $y_t$ and hence $\mathcal W_{t+1}$ through Eq.~\eqref{eq:belief}, the control affects not only the scene but also the uncertainty of future beliefs. This is the dual effect~\citep{feldbaum1960dual,barshalom1974dual}, and it makes Eq.~\eqref{eq:dual} a dual control problem in which an action is valuable both for the progress it makes and for the uncertainty it removes.
\dworld{} does not solve Eq.~\eqref{eq:dual} numerically. \plan{} approximates it by reasoning over twins sampled from $\rho_t$, with the confidence $c_t$ of Sec.~\ref{sec:plan} as its estimate of the probability that a chunk succeeds, while the harness enforces the acceptance and gate conditions above.

\section{D-World}
\label{sec:method}

\sense{}, \plan{} and the harness realize the loop of Sec.~\ref{sec:problem} as summarized in Algorithm~\ref{alg:spa}.
Each module runs as a coding agent with its own context, and the two exchange the twin, revisions and measurement evidence through the harness.

\begin{algorithm}[t]
\caption{The \dworld{} loop. $\bar{\mathcal D}$ is the latest record since reconstruction, $A(S)$ holds when the revisions $S$ satisfy Eq.~\eqref{eq:accept} on $\bar{\mathcal D}$, and $\oplus$ applies them. Refusals and \sense{}'s objections are returned to \plan{} through $h$, revisions take effect with the executed chunk, and two consecutive rounds without an executed chunk end the episode.}
\label{alg:spa}
\small
\begin{algorithmic}[1]
\Require task $\ell$, repair budget $R$
\State $\hat{\mathcal T}\gets\textsc{Sense}_\pi(\textsc{Observe}(),\ell)$\quad $I,h\gets\varnothing$\quad $\nu\gets\textsc{continue}$
\Loop
  \IfThen{$\nu=\textsc{reconstruct}$ or the operator interrupts}{$\hat{\mathcal T}\gets\textsc{Sense}_\pi(\textsc{Observe}(),\ell,\hat{\mathcal T},h)$\quad $\nu\gets\textsc{continue}$}
  \State $(\delta,\hat{\mathcal T}^{+})\gets$ the next listed chunk and $\hat{\mathcal T}$, if it passes continuation and $\mathrm G$, else $\bot$ \Comment{no planning call}
  \For{$a=0,\ldots,R$ \textbf{while} $\delta=\bot$} \Comment{\plan{}, $a>0$ is a repair}
    \State $\delta\gets\mu_\pi(\ell,I,\hat{\mathcal T},h)$\quad $\hat{\mathcal T}^{+}\gets\Call{RSI}{\hat{\mathcal T},\delta.\Delta}$
    \IfThen{$\mathrm G(\hat{\mathcal T}^{+},\delta)=0$}{$\delta\gets\bot$}
  \EndFor
  \IfThen{$\delta=\bot$}{\textbf{continue}} \Comment{nothing executes}
  \State $y\gets\textsc{Robot}(\delta.u)$\quad $I\gets\big(I,(\hat{\mathcal T}^{+},\delta.u,y)\big)$\quad $\hat{\mathcal T}\gets\textsc{Advance}(\hat{\mathcal T}^{+},\delta.u,y)$\quad $\nu\gets\delta.\nu$ \Comment{\act{}}
  \IfThen{$\nu=\textsc{done}$ and $\textsc{Check}_\pi(\textsc{Observe}(),\ell)$}{\Return success} \Comment{\sense{}}
\EndLoop
\Statex
\Function{RSI}{$\hat{\mathcal T},\Delta$} \Comment{$\Delta^{\mathrm{ind}}$ holds the changes that pass alone}
  \State \Return $\hat{\mathcal T}\oplus\Delta$ if $A(\Delta)$, else $\hat{\mathcal T}\oplus\Delta^{\mathrm{ind}}$ if $A(\Delta^{\mathrm{ind}})$, else $\hat{\mathcal T}$
\EndFunction
\end{algorithmic}
\end{algorithm}

\subsection{Writing the Twin}
\label{sec:sense}
\sense{} receives the images, the instruction, the robot's state and the robot model with its cameras. It writes the twin as an executable simulation model covering what the robot could reach or collide with, together with the surfaces these objects rest on.
The multimodal agent interprets the images and writes classical vision and geometry code to measure them, without additional pretrained perception modules.
It renders its draft from every camera for comparison with the images and the wrist depth, and a physics check reports what the draft implies before anything runs, such as interpenetrating or unsupported bodies, or bodies that move while the robot holds still.
Corrections must come from the images rather than from moving a body to pass the check.
With every estimate, \sense{} lists the quantities it is uncertain about, with reasons and with ranges that a sampled twin can instantiate, and these ranges initialize the intervals $\bar\Theta_{k,t}$.
It answers \plan{}'s questions and may request a viewpoint, which \plan{} reaches with a gated chunk. Its scripts and notes are passed to \plan{}, which can rerun and extend its measurements.

\subsection{Planning and Acting}
\label{sec:plan}
\plan{} performs LLM-guided programmatic search. It writes trajectory candidates, simulates them, inspects the resulting motion and contacts, and uses this feedback to propose further candidates within the rollout budget (App.~\ref{app:constants}).
Among the candidates it considers, \plan{} selects the action it judges most likely to advance and ultimately complete the task on the real robot, taking into account the errors that may remain in the twin.
This judgment combines simulated motion and contacts, real observations, measured residuals and \sense{}'s reported uncertainties, evidence that no single fixed cost function encodes.
With each chunk, \plan{} states a \emph{confidence} $c_t\in[0,1]$, its estimate, after the rollouts and given the residuals the twin still shows, of how likely the chunk is to achieve its expectations $\mathcal P_t$ on the real robot. The confidence is recorded and passed to \sense{}, and it does not gate execution.
\plan{} can also test a candidate in twins instantiated from $\widetilde{\mathcal W}_t$. Agreement across these samples favors actions that tolerate the tested parameter variations, whereas disagreement points to an informative interaction before a sensitive action is committed.

At the next planning call, \plan{} receives measured and predicted motion, residuals, the outcomes of measurable expectations, resistance events and the latest observation.
A language interruption by the operator requests reconstruction before further planning.

\subsection{Recursive Self-Improvement of the Twin}
\label{sec:rsi}

\begin{figure}[t]
\centering
\includegraphics[width=\linewidth]{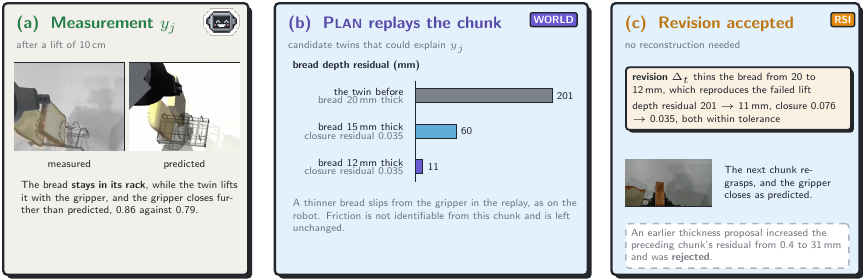}
\caption{\textbf{Replay-checked revision after a failed lift.} In a failed bread grasp, the real bread remains in its rack while the twin lifts it with the gripper. Reducing the modeled thickness reproduces the slip during replay and reduces both the depth and the closure residual. Friction, to which this chunk is insensitive, is left unchanged. The revision passes the non-regression check and is accepted, and the next chunk attempts another grasp without reconstruction. Here $j$ is the latest executed round and $t$ the revision round. An earlier thickness proposal had increased the residual on the preceding chunk and was rejected.}
\label{fig:rsi}
\end{figure}

Under \rsi{}, the twin's states and parameters evolve with the evidence, and its structure changes only through reconstruction.
Every proposed change names a supported field, a new value, a reason and its measurement evidence.
A measured discrepancy calls for one of three kinds of change, which answer different questions.
A state correction answers where the bodies are now. It synchronizes the twin, and for bodies moved by the chunk and for joints it is excluded from replay, because it does not explain the preceding transition.
A parameter revision answers why the transition differed from the prediction, and it must reproduce the measured response under replay.
A reconstruction answers that the representation itself is wrong, for example because a body is missing or an object no longer matches its modeled form.
\rsi{} performs the first two kinds of change, and the third returns the scene to \sense{}.

For parameter fitting, \plan{} uses rollout search or a least-squares tool that replays recorded controls through the simulator. The planner and the acceptance check share the same residual functions.
Accepted changes are marked \emph{fitted} and carry their evidence, so that every revised value records its provenance and \sense{} can keep, correct or undo fitted values in a later reconstruction.

Interaction supplies evidence that a static image cannot.
A wrist that stops short of its free-motion prediction indicates resistance, gripper closure constrains object size and compliance, and object motion separates mass and friction hypotheses.
When the evaluated signals are insensitive to a parameter, the replay check leaves it unconstrained, as with friction in Fig.~\ref{fig:rsi}. Resolving it requires an interaction whose outcome depends on it, which is the subject of Sec.~\ref{sec:theory}.

\subsection{The Harness}
\label{sec:harness}
The harness is the software between the two agents and the robot, and it follows one principle. Judgment belongs to the agents, while the harness supplies mechanism, measurement and hard checks.
It plays three roles that correspond to the components of agent harnesses in prior work~\citep{lee2026rhsi,wang2026regenharness}.

\paragraph{Context and tools.}
The harness defines what crosses module boundaries, namely the twin as an executable scene with its uncertainty ranges, revisions as named field assignments with evidence, chunks with their expectations $\mathcal P_t$, and a question channel from \plan{} to \sense{}.
It also maintains each agent's long-running context and sandboxed workspace (App.~\ref{app:harness}).

\paragraph{Verification.}
The harness selects revisions, gates chunks and requires \sense{}'s confirmation of completion, as stated in Proposition~\ref{prop:rsi} below. During execution, it compares measured end-effector positions and gripper closures with their predictions and halts motion at the monitor stop threshold (App.~\ref{app:constants}).

\paragraph{Execution records.}
The harness stores every round in full (App.~\ref{app:harness}), and these records serve online replay, offline refitting (Sec.~\ref{sec:library}) and the analyses of Sec.~\ref{sec:experiments}.

\begin{proposition}[Harness invariants]
\label{prop:rsi}
Under Algorithm~\ref{alg:spa}, every executed chunk has passed the gate $\mathrm G$ in the twin in which it executes, every revision applied by \rsi{} satisfies Eq.~\eqref{eq:accept} on the latest record since reconstruction, with moved-body and joint-state corrections excluded from replay, and success is returned only after \sense{} confirms completion on an observation taken after the last executed chunk.
\end{proposition}
The proof, by inspection of Algorithm~\ref{alg:spa}, is given in App.~\ref{app:dual}.
These invariants tie what reaches the robot to what the twin predicts and what the measurements support. Transfer of the predicted outcome to the physical scene additionally depends on the fidelity of the twin, and App.~\ref{app:transfer} gives a sufficient condition.

\subsection{Informative Actions and Turn-Off}
\label{sec:theory}
The two results below make operational the two values that a chunk has under Eq.~\eqref{eq:dual}, progress and information, with proofs in App.~\ref{app:dual}.

\paragraph{Which actions are informative.}
The following lemma states when the predictions of two sampled twins guarantee that executing a chunk shrinks the belief.
\begin{lemma}[Discriminating chunks]
\label{lem:disc}
Let $\theta,\theta'\in\mathcal W_t$, let the chunk $u_t$ execute in full, and let $i$ be an entry that is measured in its record $\mathcal D_t$, whose discrepancy $d_i$ is a pseudometric and whose tolerance $\varepsilon_{i,t}$ does not depend on the measurement. If $d_i\big(\hat y_i(\theta;\mathcal D_t),\hat y_i(\theta';\mathcal D_t)\big)>2\varepsilon_{i,t}$, then for every measured value at least one of $\theta,\theta'$ is not in $\mathcal W_{t+1}$.
\end{lemma}
The operational tolerance that bounds revision acceptance thus also defines what counts as an informative interaction, and \plan{}'s comparison of sampled twins tests the lemma's condition, following the prospective reasoning of dual control~\citep{barshalom1976caution}.
A test grasp separates thickness hypotheses through closure, and moving a wrist camera closer reduces the depth tolerance and sharpens the comparison of depth predictions across sampled twins, although the depth entry lies outside the lemma's hypotheses (App.~\ref{app:dual}). Many such chunks also advance the task.

\paragraph{When probing pays, and why local revision counteracts turn-off.}
Turn-off is the opposite of probing. A planner that keeps choosing views or small steps that avoid the uncertain contact never executes a chunk that meets the condition of Lemma~\ref{lem:disc} for the parameter that made the contact uncertain. No measurement is then guaranteed to reduce that uncertainty, and caution can reinforce itself~\citep{mookerjee1985turnoff}.
Whether probing is worth its price depends on how the probe's outcome enters the twin.
Consider a sensitive chunk that succeeds with probability $q$ under $\rho_t$ and otherwise incurs a failure cost $C_{\mathrm{fail}}$, and a discriminating probe that can precede it.
The probe takes planning and execution times $T'_P$ and $T'_A$ and costs $C_{\mathrm{probe}}=T'_P+T'_A+C_{\mathrm{inc}}$, where $C_{\mathrm{inc}}$ is the cost of incorporating its outcome into the twin, which is one replay-checked revision under \rsi{} or a sensing time $T_S$ under reconstruction. The sensitive chunk takes $T_P$ and $T_A$, and $C_{\mathrm{fail}}>0$ is the expected additional time after its failure, including recovery.
\begin{proposition}[Probing threshold]
\label{prop:probe}
Suppose that the probe itself does not fail and that, after it, the sensitive chunk succeeds. Then probing first has strictly lower expected time than committing the sensitive chunk directly if and only if $q<q^\star=1-C_{\mathrm{probe}}/C_{\mathrm{fail}}$. The threshold $q^\star$ is strictly decreasing in $C_{\mathrm{inc}}$, so a cheaper incorporation raises $q^\star$ and widens the interval $q<q^\star$ in which probing is preferred.
\end{proposition}
In the towel episode of Fig.~\ref{fig:real}b, a reconstruction took about 20\,min, whereas a planning call took 3 to 6\,min and a chunk seconds, so incorporating a probe's outcome by a replay-checked revision instead of a reconstruction removes the dominant term of $C_{\mathrm{probe}}$.
Reconstruction remains preferable when one new view resolves several uncertainties at once, which amounts to a larger information gain per unit of $C_{\mathrm{inc}}$, and this tradeoff is left to the planner.

\subsection{Decoupled Modules and Scaling}
\label{sec:talk}
\sense{} and \plan{} share observations and history but exchange only explicit representations, namely the twin with its uncertainty ranges, its revisions, expectations, measured evidence, and notes and questions in language.
A learned reconstruction model can therefore supply the twin and its uncertainty ranges, and a learned policy can propose trajectories that are evaluated in its simulation, while the harness continues to check revisions, gate chunks and require confirmation of completion.

The same structure determines how the loop scales.
Perception and planning inference dominate the critical path of Eq.~\eqref{eq:cost}, while robot motion occupies a small share of it (Fig.~\ref{fig:real}b).
Faster models therefore shorten each round directly and make shorter chunks affordable, more capable models write more accurate initial twins that need fewer revisions, and overlapping perception with planning can remove part of the remaining serial time (App.~\ref{sec:spec}).

\section{Experiments}
\label{sec:experiments}

\subsection{Setup}
\label{sec:setup}

\paragraph{Robot and implementation.}
The robot is a dual-arm manipulator with an RGB head camera and two RGB-D wrist cameras.
Both modules are coding agents using a frontier LLM, each in a long-running sandboxed conversation, connected through the harness (Sec.~\ref{sec:harness}).
The robot model, controller, reference camera intrinsics and mounting transforms, and metric depth are supplied platform inputs.
The agents receive no object models and undergo no task-specific training or preparatory trials.

\paragraph{Tasks.}
\emph{Towel handling} asks the robot to pick up a towel and put it aside, with the initial towel represented as a deformable body.
\emph{Toaster manipulation} asks it to put bread into a toaster slot and press the lever.

\paragraph{Measurements.}
Ground-truth geometry and physical parameters are unavailable on the real robot, so the twin is evaluated by its agreement with measurements, namely rendered-depth residuals (App.~\ref{app:error}), gripper closure and end-effector motion.
The object-level entries in Table~\ref{tab:rev} take the maximum of the recorded residuals over the available depth cameras.

\subsection{Observed Twin Revisions}
\label{sec:results}

\begin{figure}[t]
\centering
\begin{tabular}{@{}l@{\hspace{3pt}}c@{}}
\raisebox{-\height}{\small\textbf{(a)}} & \raisebox{-\height}{\includegraphics[width=0.965\linewidth]{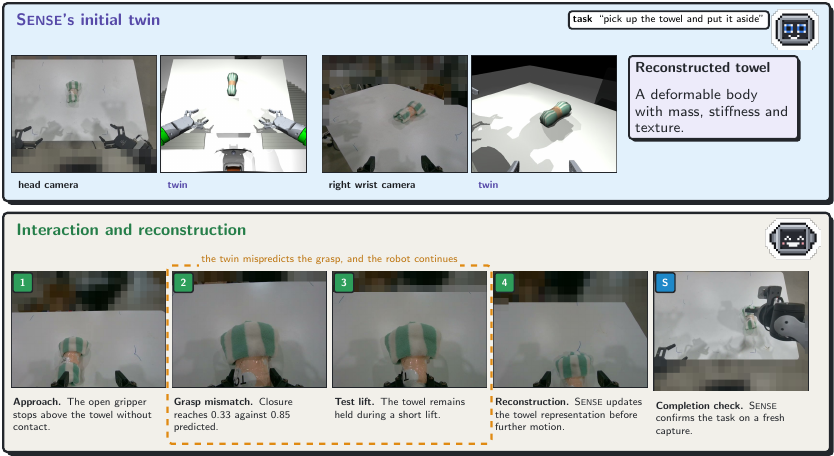}}\\[4pt]
\raisebox{-\height}{\small\textbf{(b)}} & \raisebox{-\height}{\includegraphics[width=0.965\linewidth]{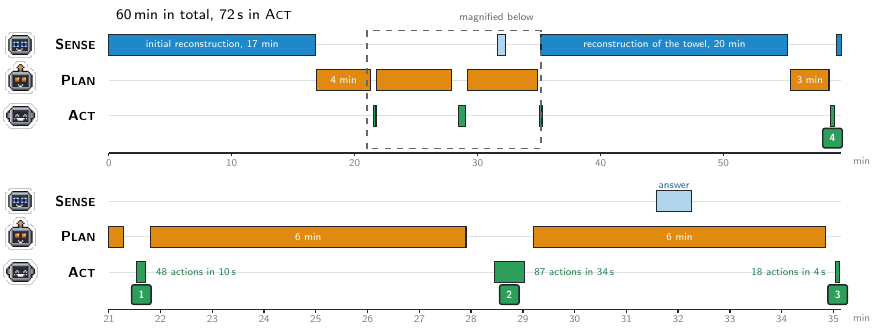}}
\end{tabular}
\caption{\textbf{Towel handling.} (a) The top row shows camera images and the initial twin rendered from the same cameras, and the bottom row selected interaction and reconstruction stages, numbered by chunk. Dashed outlines mark the grasp mismatch between the twin and the robot, after which reconstruction changes the towel representation to match its observed configuration. (b) Where the time goes in the same episode. \sense{} is blue, lighter while it answers a question, \plan{} is orange, and the robot's chunks are green with the numbers of (a). The bottom lanes magnify the three chunks between the two reconstructions. Model inference fills almost the entire hour, whereas the robot moves for seconds per chunk.}
\label{fig:real}
\end{figure}

\begin{table}[t]
\centering\footnotesize
\setlength{\tabcolsep}{4pt}
\begin{tabularx}{\linewidth}{@{}l>{\raggedright\arraybackslash}Xlrr@{}}
\toprule
Case & Revision & Residual & Before & After \\
\midrule
Slot support & support inside the slot raised by 2\,cm & bread depth & 22.4\,mm & 1.5\,mm \\
Towel geometry & table aligned with the measured plane and towel thickness scaled by 1.45 & towel depth & 60.1\,mm & 4.3\,mm \\
Bread grasp & bread thickness reduced from 20 to 12\,mm & gripper closure & 0.076 & 0.035 \\
 & & bread depth & 201\,mm & 11\,mm \\
\bottomrule
\end{tabularx}
\vspace{4pt}
\caption{\textbf{Replay residuals before and after parameter revision.} Depth entries are absolute signed-median residuals, maximized over the available cameras. Closure is dimensionless on a scale from 0 (open) to 1 (closed).}
\label{tab:rev}
\end{table}

\paragraph{State, parameters and structure.}
Parameter revisions bring the twin into agreement with what the robot measured (Table~\ref{tab:rev}).
Thinning the bread from 20 to 12\,mm reproduces the failed lift and reduces both the depth and the closure residual (Fig.~\ref{fig:rsi}), and raising the support inside the toaster slot by 2\,cm reduces the bread's depth residual from 22.4 to 1.5\,mm.
In each case of Table~\ref{tab:rev}, the next chunk executed without reconstruction.
When no parameter explains the evidence, as when the grasped towel no longer matches its initial representation, \plan{} requests reconstruction and \sense{} rewrites the towel (Fig.~\ref{fig:real}a).

\subsection{Sources of Stalled Execution}
\label{sec:turnoff}
The episode records separate three sources of slow progress.
Turn-off appears as repeated views or small steps that avoid the contact needed to resolve uncertainty.
Reconstruction adds latency, which is wasted when a local revision would have sufficed. In the towel episode, the two reconstructions together took longer than all planning calls of the episode (Fig.~\ref{fig:real}b), which is the cost that local revision avoids whenever it suffices.
A gate refusal, in contrast, is a failed model-based check rather than reluctance to act.
The stated confidence $c_t$ does not separate these cases, because it estimates whether one chunk succeeds rather than what the chunk reveals, and missed expectations also occur at high confidence.
The informative quantity is the measured discrepancy against a stated prediction, together with whether the next action can resolve it, and this is precisely the quantity on which \rsi{} operates.

\section{From Episodes to High-Fidelity Worlds}
\label{sec:library}

\begin{figure}[t]
\centering
\includegraphics[width=\linewidth]{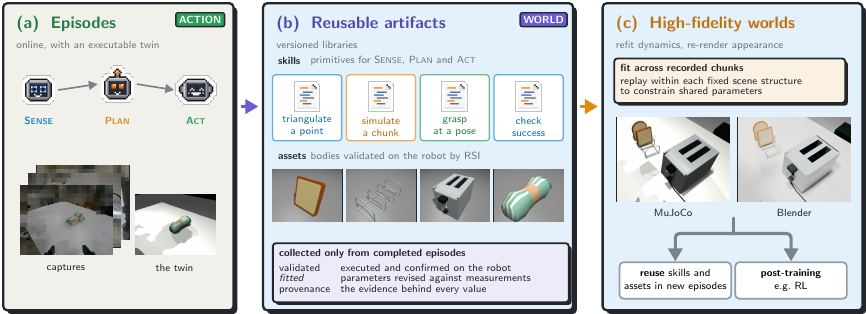}
\caption{\textbf{From episodes to high-fidelity worlds.} (a) Each episode records observations, twins, controls and revisions. (b) Versioned libraries collect agent-written skills, such as perception and motion primitives, and scene bodies from completed episodes, whose parameters RSI has validated against the robot's measurements. (c) Offline, recorded chunks are replayed jointly to refit parameters, and the refitted twin is rendered again with the same geometry, here in MuJoCo and in Blender. Skills and assets can be reused in later episodes, and refitted twins can serve post-training.}
\label{fig:library}
\end{figure}

An explicit world outlives the episode in which it was written.
Its records hold every round of interaction (Fig.~\ref{fig:library}), and the skills and scene bodies of completed episodes are collected into versioned libraries.
Because the bodies of such a twin have been tested against the robot's measurements and revised by \rsi{} where their predictions disagreed, the library contains assets whose geometry and physical parameters are consistent with real measurements, together with the provenance of each value.

Online, \rsi{} fits each revision to the latest chunk, which suffices for acting.
Offline, the same records support a second optimization. Within each interval of fixed scene structure, all recorded chunks can be replayed jointly to constrain shared parameters, and the refitted twin can be rendered at higher fidelity with unchanged geometry.
This extends the real-to-sim-to-real cycle to the level of the dataset.
The resulting twins are reconstructions of real scenes grounded in real robot measurements, which makes them natural environments for post-training, for example reinforcement learning of policies that are subsequently deployed on the robot.

\section{Limitations}
\label{sec:limitations}
Inference latency currently restricts \dworld{} to quasi-static tasks. The twin is limited to the mechanisms its simulator represents, and neither module adds missing physics as residual-engine approaches do~\citep{liang2026empiric}.
Replay acceptance is local. It enforces non-regression on the latest chunk within operational tolerances, without guaranteeing improvement on earlier chunks or identifying unique physical parameters, and the implemented consistency intervals approximate the belief set coordinate-wise.
The gate and monitor check one nominal twin and the planner's expectations, without bounding contact forces.
\plan{} approximates the dual-control objective by judgment over sampled twins rather than by optimization. Further experiments, including systematic comparisons with reconstruction-only and learned baselines, remain to be run.

\section{Conclusion}
\label{sec:conclusion}
\dworld{} treats the world model of a physical agent as an explicit, executable artifact that the agent writes, tests against the robot's measurements and improves while it acts.
This choice changes what a prediction error means.
In an implicit model, an error is a symptom that further training may or may not remove. In an explicit twin, it points to named fields, and a single replay decides whether a proposed explanation is consistent with what the robot measured.
Recursive self-improvement builds on this property. State corrections keep the twin synchronized with the scene and replay-checked parameter revisions make its predictions agree with the latest response, while reconstruction changes its structure when no parameter can explain the evidence.

The same structure makes the loop a dual controller.
Every executed chunk both advances the task and constrains the twin. Because one local revision absorbs the outcome of an informative interaction, probing costs far less than reconstruction, and the planner has a reason to resolve uncertainty rather than avoid it.
The harness gives these judgments the guarantees of Proposition~\ref{prop:rsi}. It defines what the agents exchange, admits only revisions and chunks that pass its checks, and records every round in a form that can be replayed.

Three properties suggest how the approach scales.
Perception and planning meet only at the twin and in language, so learned reconstruction models and policies can replace the agents without changing the checks.
Inference dominates the loop, so faster and stronger models translate directly into shorter rounds, shorter chunks and more accurate initial twins, which brings dynamic interaction within reach.
Finally, every episode leaves behind reconstructed scenes grounded in real measurements, which offline refitting can raise to higher fidelity. Accumulated across tasks and robots, these twins form a growing set of real-to-sim environments in which policies can be post-trained, for example with reinforcement learning, and returned to the robot, which completes the real-to-sim-to-real cycle beyond the episode.

\clearpage

\begingroup
\setlength{\bibsep}{4pt plus 1pt minus 1pt}
\bibliography{refs}

\begin{thebibliography}{75}
\providecommand{\natexlab}[1]{#1}
\providecommand{\url}[1]{\texttt{#1}}
\expandafter\ifx\csname urlstyle\endcsname\relax
  \providecommand{\doi}[1]{doi: #1}\else
  \providecommand{\doi}{doi: \begingroup \urlstyle{rm}\Url}\fi

\bibitem[Liang et~al.(2023)Liang, Huang, Xia, Xu, Hausman, Ichter, Florence,
  and Zeng]{liang2023cap}
J.~Liang, W.~Huang, F.~Xia, P.~Xu, K.~Hausman, B.~Ichter, P.~Florence, and
  A.~Zeng.
\newblock {Code as Policies: Language Model Programs for Embodied Control}.
\newblock In \emph{IEEE International Conference on Robotics and Automation
  (ICRA)}, 2023.

\bibitem[Fu et~al.(2026)Fu, Yu, El-Refai, Kou, Xue, Huang, Xiao, Wang, Niu, Li,
  Shi, Wu, Sastry, Zhu, Goldberg, and Fan]{fu2026capx}
L.~Fu, J.~Yu, K.~El-Refai, E.~Kou, H.~Xue, H.~Huang, W.~Xiao, G.~Wang, D.~Niu,
  F.-F. Li, G.~Shi, J.~Wu, S.~Sastry, Y.~Zhu, K.~Goldberg, and L.~J. Fan.
\newblock {CaP-X: A Framework for Benchmarking and Improving Coding Agents for
  Robot Manipulation}.
\newblock In \emph{International Conference on Machine Learning (ICML)}, 2026.

\bibitem[Zhang et~al.(2026)Zhang, Wang, Ouyang, Huang, Li, Su, Jin, Chai,
  Liang, Dou, Chen, and Chen]{zhang2026astrapolicy}
W.~Zhang, K.~Wang, Y.~Ouyang, X.~Huang, L.~Li, K.~Su, W.~Jin, W.~Chai,
  H.~Liang, Z.~Dou, Y.~Chen, and T.~Chen.
\newblock {An Unexpected Robot Policy: Early Evaluations of GPT-6 Astra on
  RoboDojo and Beyond}.
\newblock \emph{arXiv preprint arXiv:2609.24170}, 2026.

\bibitem[{Galbot Team}(2026)]{galbot2026astra}
{Galbot Team}.
\newblock {Systematically Exploring the Capabilities of GPT-6 Astra as Embodied
  Policies}.
\newblock \emph{arXiv preprint arXiv:2609.38537}, 2026.

\bibitem[Chen et~al.(2026)Chen, Chen, Li, Tang, Su, Lu, Wan, Chen, Liu, Yan,
  Su, Dou, Wang, Zhang, Liu, Qin, Liang, Wu, Lin, Lin, Wang, He, Wu, Wu, Zhou,
  Lei, Yu, Ji, Jin, Lin, Li, Xiong, Xu, Li, Chai, Xie, Wang, Mu, Dong, Matusik,
  Ding, Ding, Luo, and Tomizuka]{chen2026robodojo}
T.~Chen, Y.~Chen, Z.~Li, J.~Tang, K.~Su, H.~Lu, W.~Wan, B.~Chen, S.~Liu,
  H.~Yan, H.~Su, Z.~Dou, K.~Wang, D.~Zhang, Y.~Liu, Y.~Qin, Q.~Liang, Q.~Wu,
  Z.~Lin, W.~Lin, Y.~Wang, M.~He, T.~Wu, R.~Wu, J.~Zhou, K.-C. Lei, H.~Yu,
  Y.~Ji, W.~Jin, G.~Lin, X.~Li, Q.~Xiong, R.~Xu, Z.~Li, W.~Chai, E.~Xie,
  Z.~Wang, Y.~Mu, H.~Dong, W.~Matusik, M.~Ding, W.~Ding, P.~Luo, and
  M.~Tomizuka.
\newblock {RoboDojo: A Unified Sim-and-Real Benchmark for Comprehensive
  Evaluation of Generalist Robot Manipulation Policies}.
\newblock \emph{arXiv preprint arXiv:2607.04434}, 2026.

\bibitem[Ye et~al.(2026)Ye, Ge, Zheng, Gao, Yu, Kurian, Indupuru, Tan, Zhu,
  Xiang, Malik, Lee, Liang, Ranawaka, Gu, Xu, Wang, Hu, Narayan, Bjorck, Wang,
  Kim, Niu, Zheng, Xie, Wu, Wang, Julian, Xu, Du, Chebotar, Reed, Kautz, Zhu,
  Fan, and Jang]{ye2026dreamzero}
S.~Ye, Y.~Ge, K.~Zheng, S.~Gao, S.~Yu, G.~Kurian, S.~Indupuru, Y.~L. Tan,
  C.~Zhu, J.~Xiang, A.~Malik, K.~Lee, W.~Liang, N.~Ranawaka, J.~Gu, Y.~Xu,
  G.~Wang, F.~Hu, A.~Narayan, J.~Bjorck, J.~Wang, G.~Kim, D.~Niu, R.~Zheng,
  Y.~Xie, J.~Wu, Q.~Wang, R.~Julian, D.~Xu, Y.~Du, Y.~Chebotar, S.~Reed,
  J.~Kautz, Y.~Zhu, L.~Fan, and J.~Jang.
\newblock {World Action Models are Zero-shot Policies}.
\newblock \emph{arXiv preprint arXiv:2602.15922}, 2026.

\bibitem[Lu et~al.(2026)Lu, Zhai, Wang, Zeng, Yang, Liu, Cheng, Zhang, Qiu,
  Cheng, Laptev, Xu, Riviere, Loianno, Xing, and Zuo]{lu2026wamsurvey}
Z.~Lu, H.~Zhai, G.~Wang, H.~Zeng, J.~Yang, J.~Liu, L.~Cheng, Y.~Zhang, Y.~Qiu,
  Z.~Cheng, I.~Laptev, D.~Xu, B.~Riviere, G.~Loianno, E.~Xing, and X.~Zuo.
\newblock {World-Action Models for Robot Learning and Control: A Survey}.
\newblock \emph{arXiv preprint arXiv:2609.16074}, 2026.

\bibitem[Yang et~al.(2026)Yang, Liu, Kou, Chen, Hu, Zhou, Zhao, Wei, Xia, Li,
  Liu, and Deng]{yang2026wla}
Y.~Yang, Z.~Liu, S.~Kou, Y.~Chen, Y.~Hu, J.~Zhou, B.~Zhao, Z.~Wei, X.~Xia,
  X.~Li, P.~Liu, and Z.~Deng.
\newblock {World-Language-Action Model for Unified World Modeling, Language
  Reasoning, and Action Synthesis}.
\newblock \emph{arXiv preprint arXiv:2606.05979}, 2026.

\bibitem[Wang et~al.(2026)Wang, Cai, Chen, Chi, Sun, Dai, Hung, Guo, Ren, Yao,
  Liu, Long, Duan, Gao, Lyu, Liu, and Wu]{wang2026codeasworld}
H.~Wang, Y.~Cai, W.~Chen, J.~Chi, H.~Sun, Q.~Dai, Y.-H. Hung, X.~Guo, J.~Ren,
  R.~Yao, Z.~Liu, M.~Long, Y.~Duan, J.~Gao, J.~Lyu, F.~Liu, and J.~Wu.
\newblock {Code as Worlds: Agentic Discovery of Executable World
  Representations for Physical Reasoning}.
\newblock \emph{arXiv preprint arXiv:2608.27549}, 2026.

\bibitem[Tang et~al.(2026{\natexlab{a}})Tang, Zhang, Yang, Yan, Zhou, Mao, Han,
  Cao, Xu, He, Fu, Feng, and Huang]{tang2026video2world}
J.~Tang, Z.~Zhang, J.~Yang, Y.~Yan, K.~Zhou, L.~Mao, R.~Han, J.~Cao, W.~Xu,
  L.~He, M.~Fu, F.~Feng, and B.~Huang.
\newblock {Video2World: Benchmarking Coding Agents for Interactive World
  Modeling from Embodied Videos}.
\newblock \emph{arXiv preprint arXiv:2610.04432}, 2026{\natexlab{a}}.

\bibitem[Tang et~al.(2026{\natexlab{b}})Tang, Shimada, Meng, Bi, Liu, Wang,
  Xiao, Tan, Zhang, Huang, Liang, Shen, Song, Vosoughi, Feng, Filvantorkaman,
  and Xu]{tang2026bvb}
Y.~Y. Tang, D.~Shimada, J.~Meng, J.~Bi, P.~Liu, Y.~Wang, Y.~Xiao, Z.~Tan,
  Z.~Zhang, C.~Huang, S.~Liang, Q.~Shen, L.~Song, A.~Vosoughi, M.~Feng,
  M.~Filvantorkaman, and C.~Xu.
\newblock {BVB: Benchmarking Agentic Video Understanding via Programmatic
  Reconstruction in Blender}.
\newblock \emph{arXiv preprint arXiv:2609.15478}, 2026{\natexlab{b}}.

\bibitem[Nilsson(1984)]{nilsson1984shakey}
N.~J. Nilsson.
\newblock {Shakey the Robot}.
\newblock Technical Report Technical Note 323, SRI International, 1984.

\bibitem[Brooks(1986)]{brooks1986robust}
R.~A. Brooks.
\newblock {A Robust Layered Control System for a Mobile Robot}.
\newblock \emph{IEEE Journal on Robotics and Automation}, 2\penalty0
  (1):\penalty0 14--23, 1986.

\bibitem[Brooks(1991)]{brooks1991intelligence}
R.~A. Brooks.
\newblock {Intelligence without Representation}.
\newblock \emph{Artificial Intelligence}, 47\penalty0 (1--3):\penalty0
  139--159, 1991.

\bibitem[Torne et~al.(2024)Torne, Simeonov, Li, Chan, Chen, Gupta, and
  Agrawal]{torne2024rialto}
M.~Torne, A.~Simeonov, Z.~Li, A.~Chan, T.~Chen, A.~Gupta, and P.~Agrawal.
\newblock {Reconciling Reality through Simulation: A Real-to-Sim-to-Real
  Approach for Robust Manipulation}.
\newblock In \emph{Robotics: Science and Systems (RSS)}, 2024.

\bibitem[Feldbaum(1960)]{feldbaum1960dual}
A.~A. Feldbaum.
\newblock {Dual Control Theory. I}.
\newblock \emph{Avtomatika i Telemekhanika}, 21\penalty0 (9):\penalty0
  1240--1249, 1960.

\bibitem[Bar-Shalom and Tse(1974)]{barshalom1974dual}
Y.~Bar-Shalom and E.~Tse.
\newblock {Dual Effect, Certainty Equivalence, and Separation in Stochastic
  Control}.
\newblock \emph{IEEE Transactions on Automatic Control}, 19\penalty0
  (5):\penalty0 494--500, 1974.

\bibitem[{\AA}str{\"o}m and Wittenmark(1971)]{astrom1971problems}
K.~J. {\AA}str{\"o}m and B.~Wittenmark.
\newblock {Problems of Identification and Control}.
\newblock \emph{Journal of Mathematical Analysis and Applications}, 34\penalty0
  (1):\penalty0 90--113, 1971.
\newblock \doi{10.1016/0022-247X(71)90161-2}.

\bibitem[Wittenmark(1995)]{wittenmark1995dual}
B.~Wittenmark.
\newblock {Adaptive Dual Control Methods: An Overview}.
\newblock \emph{IFAC Proceedings Volumes}, 28:\penalty0 67--72, 1995.

\bibitem[{Physical Intelligence} et~al.(2025){Physical Intelligence}, Black,
  Brown, Darpinian, Dhabalia, Driess, Esmail, Equi, Finn, Fusai, Galliker,
  Ghosh, Groom, Hausman, Ichter, Jakubczak, Jones, Ke, LeBlanc, Levine,
  Li-Bell, Mothukuri, Nair, Pertsch, Ren, Shi, Smith, Springenberg, Stachowicz,
  Tanner, Vuong, Walke, Walling, Wang, Yu, and Zhilinsky]{pi2025pi05}
{Physical Intelligence}, K.~Black, N.~Brown, J.~Darpinian, K.~Dhabalia,
  D.~Driess, A.~Esmail, M.~Equi, C.~Finn, N.~Fusai, M.~Y. Galliker, D.~Ghosh,
  L.~Groom, K.~Hausman, B.~Ichter, S.~Jakubczak, T.~Jones, L.~Ke, D.~LeBlanc,
  S.~Levine, A.~Li-Bell, M.~Mothukuri, S.~Nair, K.~Pertsch, A.~Z. Ren, L.~X.
  Shi, L.~Smith, J.~T. Springenberg, K.~Stachowicz, J.~Tanner, Q.~Vuong,
  H.~Walke, A.~Walling, H.~Wang, L.~Yu, and U.~Zhilinsky.
\newblock {$\pi_{0.5}$: a Vision-Language-Action Model with Open-World
  Generalization}.
\newblock In \emph{Conference on Robot Learning (CoRL)}, 2025.

\bibitem[{NVIDIA} et~al.(2025){NVIDIA}, Bjorck, Castañeda, Cherniadev, Da,
  Ding, Fan, Fang, Fox, Hu, Huang, Jang, Jiang, Kautz, Kundalia, Lao, Li, Lin,
  Lin, Liu, Llontop, Magne, Mandlekar, Narayan, Nasiriany, Reed, Tan, Wang,
  Wang, Wang, Wang, Xiang, Xie, Xu, Xu, Ye, Yu, Zhang, Zhang, Zhao, Zheng, and
  Zhu]{nvidia2025groot}
{NVIDIA}, J.~Bjorck, F.~Castañeda, N.~Cherniadev, X.~Da, R.~Ding, L.~J. Fan,
  Y.~Fang, D.~Fox, F.~Hu, S.~Huang, J.~Jang, Z.~Jiang, J.~Kautz, K.~Kundalia,
  L.~Lao, Z.~Li, Z.~Lin, K.~Lin, G.~Liu, E.~Llontop, L.~Magne, A.~Mandlekar,
  A.~Narayan, S.~Nasiriany, S.~Reed, Y.~L. Tan, G.~Wang, Z.~Wang, J.~Wang,
  Q.~Wang, J.~Xiang, Y.~Xie, Y.~Xu, Z.~Xu, S.~Ye, Z.~Yu, A.~Zhang, H.~Zhang,
  Y.~Zhao, R.~Zheng, and Y.~Zhu.
\newblock {GR00T N1: An Open Foundation Model for Generalist Humanoid Robots}.
\newblock \emph{arXiv preprint arXiv:2503.14734}, 2025.

\bibitem[Yuan et~al.(2026)Yuan, Dong, Liu, and Zhao]{yuan2026fastwam}
T.~Yuan, Z.~Dong, Y.~Liu, and H.~Zhao.
\newblock {Fast-WAM: Do World Action Models Need Test-time Future Imagination?}
\newblock \emph{arXiv preprint arXiv:2603.16666}, 2026.

\bibitem[Tang et~al.(2024)Tang, Key, and Ellis]{tang2024worldcoder}
H.~Tang, D.~Key, and K.~Ellis.
\newblock {WorldCoder, a Model-Based LLM Agent: Building World Models by
  Writing Code and Interacting with the Environment}.
\newblock In \emph{Advances in Neural Information Processing Systems
  (NeurIPS)}, 2024.

\bibitem[Dainese et~al.(2024)Dainese, Merler, Alakuijala, and
  Marttinen]{dainese2024gifmcts}
N.~Dainese, M.~Merler, M.~Alakuijala, and P.~Marttinen.
\newblock {Generating Code World Models with Large Language Models Guided by
  Monte Carlo Tree Search}.
\newblock In \emph{Advances in Neural Information Processing Systems
  (NeurIPS)}, 2024.

\bibitem[Li et~al.(2026{\natexlab{a}})Li, Chen, and
  Alonso-Mora]{li2026explicitwm}
X.~Li, G.~Chen, and J.~Alonso-Mora.
\newblock {Building Explicit World Model for Zero-Shot Open-World Object
  Manipulation}.
\newblock \emph{arXiv preprint arXiv:2603.13825}, 2026{\natexlab{a}}.

\bibitem[Li et~al.(2026{\natexlab{b}})Li, Liao, and Zhu]{li2026rcwm}
Z.~Li, Y.~Liao, and B.~Zhu.
\newblock {Recursive Code World Models: Building Complex Worlds through
  Recursive Scene Programs}.
\newblock \emph{arXiv preprint arXiv:2609.11499}, 2026{\natexlab{b}}.

\bibitem[Chen et~al.(2026)Chen, Lin, and Zhang]{chen2026codeworldmodel}
Y.~Chen, G.~Lin, and C.~Zhang.
\newblock {Code World Model: Coding Agent as World Brain}.
\newblock \emph{arXiv preprint arXiv:2608.25927}, 2026.

\bibitem[Bai and Xiong(2026)]{bai2026visualpatchworld}
J.~Bai and J.~Xiong.
\newblock {VisualPatchWorld: Code World Models as Latent Structured
  Representations for Planning}.
\newblock \emph{arXiv preprint arXiv:2607.25236}, 2026.

\bibitem[Ning et~al.(2025)Ning, Fang, and Ma]{ning2025pwtf}
C.~Ning, K.~Fang, and W.-C. Ma.
\newblock {Prompting with the Future: Open-World Model Predictive Control with
  Interactive Digital Twins}.
\newblock In \emph{Robotics: Science and Systems (RSS)}, 2025.

\bibitem[Liu et~al.(2026)Liu, Yao, Chen, Gao, Mao, Huang, and
  Du]{liu2026simpact}
H.~Liu, S.~Yao, H.~Chen, J.~Gao, J.~Mao, J.-B. Huang, and Y.~Du.
\newblock {SIMPACT: Simulation-Enabled Action Planning using Vision-Language
  Models}.
\newblock In \emph{IEEE/CVF Conference on Computer Vision and Pattern
  Recognition (CVPR)}, 2026.

\bibitem[Abou-Chakra et~al.(2026)Abou-Chakra, Sun, Rana, May, Schmeckpeper,
  Suenderhauf, Minniti, and Herlant]{abouchakra2025realissim}
J.~Abou-Chakra, L.~Sun, K.~Rana, B.~May, K.~Schmeckpeper, N.~Suenderhauf, M.~V.
  Minniti, and L.~Herlant.
\newblock {Real-is-Sim: Bridging the Sim-to-Real Gap with a Dynamic Digital
  Twin}.
\newblock In \emph{IEEE International Conference on Robotics and Automation
  (ICRA)}, 2026.

\bibitem[Wang et~al.(2026)Wang, Tian, Swann, Shorinwa, Wu, and
  Schwager]{wang2026phys2real}
M.~Wang, S.~Tian, A.~Swann, O.~Shorinwa, J.~Wu, and M.~Schwager.
\newblock {Phys2Real: Fusing VLM Priors with Interactive Online Adaptation for
  Uncertainty-Aware Sim-to-Real Manipulation}.
\newblock In \emph{IEEE International Conference on Robotics and Automation
  (ICRA)}, 2026.

\bibitem[Hu et~al.(2026)Hu, Mart{\'i}n-Mart{\'i}n, Stone, Duan, Jiang, and
  Shi]{hu2026simex}
J.~Hu, R.~Mart{\'i}n-Mart{\'i}n, P.~Stone, R.~Duan, Z.~Jiang, and G.~Shi.
\newblock {SimEX: Simulation-Integrated Robotics AutoResearch}.
\newblock \emph{arXiv preprint arXiv:2609.38982}, 2026.

\bibitem[Chen et~al.(2026)Chen, Xia, Peng, Zhang, Jing, Ma, Qian, Cheng, Jiao,
  Zhou, Qu, Ye, Zhang, Wang, Zeng, Chen, Yang, Zeng, Luo, Wang, Liu, Yuille,
  Shi, Zheng, Li, Jiang, and Chen]{chen2026agenticreal2sim}
G.~Chen, Q.~Xia, J.~Peng, H.~Zhang, P.~Jing, B.~Ma, J.~Qian, Y.~Cheng, Z.~Jiao,
  B.~Zhou, Y.~Qu, L.~Ye, K.~Zhang, K.~Wang, W.~Zeng, Y.~Chen, P.~Yang, Z.~Zeng,
  S.~Luo, H.~Wang, C.~Liu, A.~Yuille, F.~Shi, C.~Zheng, Y.~Li, C.~Jiang, and
  P.~Y. Chen.
\newblock {Agentic Real2Sim: Physics-based World Modeling with Vision-Language
  Agents}.
\newblock \emph{arXiv preprint arXiv:2607.19190}, 2026.

\bibitem[Li et~al.(2026)Li, Feng, Sun, Chen, Ren, Yang, Xu, Wu, Cheng, and
  Xu]{li2026agenticrsr}
Y.~Li, Y.~Feng, S.~Sun, J.~Chen, H.~Ren, B.~Yang, W.~Xu, Q.~Wu, H.~Cheng, and
  R.~Xu.
\newblock {Agentic RSR: Real-to-Sim-to-Real through Scene Reconstruction and
  Execution-Grounded Robot Policies}.
\newblock \emph{arXiv preprint arXiv:2610.10479}, 2026.

\bibitem[Huang et~al.(2023)Huang, Wang, Zhang, Li, Wu, and
  Fei-Fei]{huang2023voxposer}
W.~Huang, C.~Wang, R.~Zhang, Y.~Li, J.~Wu, and L.~Fei-Fei.
\newblock {VoxPoser: Composable 3D Value Maps for Robotic Manipulation with
  Language Models}.
\newblock In \emph{Conference on Robot Learning (CoRL)}, 2023.

\bibitem[Huang et~al.(2024)Huang, Wang, Li, Zhang, and Fei-Fei]{huang2024rekep}
W.~Huang, C.~Wang, Y.~Li, R.~Zhang, and L.~Fei-Fei.
\newblock {ReKep: Spatio-Temporal Reasoning of Relational Keypoint Constraints
  for Robotic Manipulation}.
\newblock In \emph{Conference on Robot Learning (CoRL)}, 2024.

\bibitem[Chu et~al.(2026)Chu, Barr, Goues, O'Hearn, Harman, Sarro, and
  Ye]{chu2026openrua}
Z.~Chu, E.~T. Barr, C.~L. Goues, P.~O'Hearn, M.~Harman, F.~Sarro, and H.~Ye.
\newblock {OpenRUA: Robot-Use Agents Are Zero-Shot Visuomotor Policies}.
\newblock \emph{arXiv preprint arXiv:2610.02459}, 2026.

\bibitem[Tsui et~al.(2026)Tsui, Fang, and Hwu]{tsui2026faea}
B.~Y. Tsui, A.~Y. Fang, and T.~J. Hwu.
\newblock {Demonstration-Free Robotic Control via LLM Agents}.
\newblock In \emph{IEEE/RSJ International Conference on Intelligent Robots and
  Systems (IROS)}, 2026.

\bibitem[Sur{\'\i}s et~al.(2023)Sur{\'\i}s, Menon, and
  Vondrick]{suris2023vipergpt}
D.~Sur{\'\i}s, S.~Menon, and C.~Vondrick.
\newblock {ViperGPT: Visual Inference via Python Execution for Reasoning}.
\newblock In \emph{IEEE/CVF International Conference on Computer Vision
  (ICCV)}, 2023.

\bibitem[Talwar et~al.(2026)Talwar, Nijs, Verleysen, and
  {wyffels}]{talwar2026localcoding}
R.~Talwar, E.~Nijs, A.~Verleysen, and F.~{wyffels}.
\newblock {Generalizing Manipulation Skills with a Local Coding Agent}.
\newblock \emph{arXiv preprint arXiv:2609.26499}, 2026.

\bibitem[Han et~al.(2026)Han, Hu, Qiu, Wu, and He]{han2026vista}
Q.~Han, K.~Hu, L.~Qiu, C.~Wu, and K.~He.
\newblock {VISTA: A Visual Harness for Reasoning in an Interactive World}.
\newblock \emph{arXiv preprint arXiv:2610.02200}, 2026.

\bibitem[Wei et~al.(2026)Wei, Tan, and Wang]{wei2026harnessevolution}
Z.~Wei, F.~Tan, and J.~Wang.
\newblock {Scale and Selection: What Makes Automatic Harness Evolution Work for
  Visual-Interface Robot Agents}.
\newblock \emph{arXiv preprint arXiv:2609.39304}, 2026.

\bibitem[Xu et~al.(2026)Xu, Shang, Dong, and Ferrara]{xu2026safeharness}
B.~Xu, Y.~Shang, Z.~Dong, and E.~Ferrara.
\newblock {Coding Agents with Harness for Safe Robot Control}.
\newblock \emph{arXiv preprint arXiv:2609.20822}, 2026.

\bibitem[Wang et~al.(2026)Wang, Jie, Yan, Heng, and Li]{wang2026regenharness}
K.~Wang, H.~Jie, Y.~Yan, Z.~Heng, and Z.~Li.
\newblock {RegenHarness: A Robot Agent Harness with Evidence-Gated Recursive
  Self-Improvement}.
\newblock \emph{arXiv preprint arXiv:2609.27612}, 2026.

\bibitem[Kambhampati et~al.(2024)Kambhampati, Valmeekam, Guan, Verma, Stechly,
  Bhambri, Saldyt, and Murthy]{kambhampati2024llmmodulo}
S.~Kambhampati, K.~Valmeekam, L.~Guan, M.~Verma, K.~Stechly, S.~Bhambri,
  L.~Saldyt, and A.~Murthy.
\newblock {Position: LLMs Can't Plan, But Can Help Planning in LLM-Modulo
  Frameworks}.
\newblock In \emph{International Conference on Machine Learning (ICML)}, 2024.

\bibitem[Hong et~al.(2026)Hong, Wei, and Zhang]{hong2026bridging}
W.~Hong, L.~Wei, and D.~Zhang.
\newblock {Bridging Semantics and Physics with Constrained LLMs for Safe and
  Trustworthy Robotic Manipulation}.
\newblock In \emph{European Conference on Computer Vision (ECCV) Workshops},
  2026.

\bibitem[Zhou et~al.(2025)Zhou, Su, Chi, Zhang, Wang, Huang, Sheng, and
  Wang]{zhou2025codeasmonitor}
E.~Zhou, Q.~Su, C.~Chi, Z.~Zhang, Z.~Wang, T.~Huang, L.~Sheng, and H.~Wang.
\newblock {Code-as-Monitor: Constraint-aware Visual Programming for Reactive
  and Proactive Robotic Failure Detection}.
\newblock In \emph{IEEE/CVF Conference on Computer Vision and Pattern
  Recognition (CVPR)}, 2025.

\bibitem[Wabersich and Zeilinger(2021)]{wabersich2018psf}
K.~P. Wabersich and M.~N. Zeilinger.
\newblock {A predictive safety filter for learning-based control of constrained
  nonlinear dynamical systems}.
\newblock \emph{Automatica}, 129:\penalty0 109597, 2021.

\bibitem[Wang et~al.(2023)Wang, Xie, Jiang, Mandlekar, Xiao, Zhu, Fan, and
  Anandkumar]{wang2023voyager}
G.~Wang, Y.~Xie, Y.~Jiang, A.~Mandlekar, C.~Xiao, Y.~Zhu, L.~Fan, and
  A.~Anandkumar.
\newblock {Voyager}: An open-ended embodied agent with large language models.
\newblock \emph{arXiv preprint arXiv:2305.16291}, 2023.

\bibitem[Yao et~al.(2023)Yao, Zhao, Yu, Du, Shafran, Narasimhan, and
  Cao]{yao2023react}
S.~Yao, J.~Zhao, D.~Yu, N.~Du, I.~Shafran, K.~Narasimhan, and Y.~Cao.
\newblock {ReAct}: Synergizing reasoning and acting in language models.
\newblock In \emph{International Conference on Learning Representations
  (ICLR)}, 2023.

\bibitem[Shinn et~al.(2023)Shinn, Cassano, Gopinath, Narasimhan, and
  Yao]{shinn2023reflexion}
N.~Shinn, F.~Cassano, A.~Gopinath, K.~Narasimhan, and S.~Yao.
\newblock {Reflexion: Language Agents with Verbal Reinforcement Learning}.
\newblock In \emph{Advances in Neural Information Processing Systems
  (NeurIPS)}, 2023.

\bibitem[Lee et~al.(2026)Lee, Xu, Seely, Lee, Zaharia, and Tang]{lee2026rhsi}
H.~Lee, J.~Xu, J.~Seely, D.~Lee, M.~Zaharia, and Y.~Tang.
\newblock {Recursive Harness Self-Improvement}.
\newblock \emph{arXiv preprint arXiv:2607.15524}, 2026.

\bibitem[Xia et~al.(2026)Xia, Han, Wang, Chen, Zhuang, Lee, Huang, Yu, CuiZhu,
  Ming, Yao, Gokturk, Pfister, and Lee]{xia2026rrsi}
P.~Xia, R.~Han, Z.~Wang, Y.~Chen, Y.~Zhuang, Y.~Lee, C.~Huang, H.~Yu,
  Z.~CuiZhu, Y.~Ming, H.~Yao, B.~Gokturk, T.~Pfister, and C.-Y. Lee.
\newblock {RRSI: Regularized Recursive Self-Improvement of Agent Harnesses}.
\newblock \emph{arXiv preprint arXiv:2609.24972}, 2026.

\bibitem[Lang et~al.(2026)Lang, Lu, Sang, Luo, Chen, Li, Yu, Shi, and
  Wang]{lang2026embodirsi}
H.~Lang, H.~Lu, S.~Sang, H.~Luo, G.~Chen, Q.~Li, J.~Yu, Y.~Shi, and J.~Wang.
\newblock {EmbodiRSI: Recursive Self-Improvement for Data-Efficient Robot
  Adaptation}.
\newblock \emph{arXiv preprint arXiv:2609.38905}, 2026.

\bibitem[Song et~al.(2026)Song, Liang, Fu, Wang, Yang, and
  Liu]{song2026embodiedrsi}
P.~Song, Z.~Liang, K.~Fu, M.~Wang, J.~Yang, and S.~Liu.
\newblock {EmbodiedRSI: Active Continual Robot Learning Through
  Hypothesis-Guided Co-Evolution}.
\newblock \emph{arXiv preprint arXiv:2610.10498}, 2026.

\bibitem[Chen et~al.(2026)Chen, Wang, and Qu]{chen2026rsisurvey}
M.~Chen, L.~Wang, and B.~Qu.
\newblock {Recursive Self-Improvement in AI: From Bounded Self-Refinement to
  Autonomous Research Loops}.
\newblock \emph{arXiv preprint arXiv:2607.07663}, 2026.

\bibitem[Wang(2026)]{wang2026whatstops}
J.~Wang.
\newblock {What Stops Recursive Self-Improvement in Robotics? Lessons from 123
  Rounds of Agentic Skill Discovery}.
\newblock \emph{arXiv preprint arXiv:2609.31760}, 2026.

\bibitem[Cherian et~al.(2026)Cherian, Corcodel, Jain, and
  Romeres]{cherian2026llmphy}
A.~Cherian, R.~Corcodel, S.~Jain, and D.~Romeres.
\newblock {LLMPhy: Parameter-Identifiable Physical Reasoning Combining Large
  Language Models and Physics Engines}.
\newblock In \emph{International Conference on Artificial Intelligence and
  Statistics (AISTATS)}, 2026.

\bibitem[Zhang et~al.(2026)Zhang, Zhang, Ng, and Deng]{zhang2026worldevolver}
X.~Zhang, W.~Zhang, S.-K. Ng, and Y.~Deng.
\newblock {Self-Evolving World Models for LLM Agent Planning}.
\newblock In \emph{Findings of the Association for Computational Linguistics:
  EMNLP}, 2026.

\bibitem[Liang et~al.(2026)Liang, Li, Nguyen, Bunnapradist, Naim, Kodali,
  Merler, Li, Gopinathan, Liu, Pimpalkhare, Tenenbaum, Weller, Tavares, Silver,
  and Ellis]{liang2026empiric}
Y.~Liang, A.~Li, D.~Nguyen, E.~Bunnapradist, M.~Naim, S.~Kodali, M.~Merler,
  B.~Li, K.~Gopinathan, Y.~Liu, N.~Pimpalkhare, J.~B. Tenenbaum, A.~Weller,
  Z.~Tavares, T.~Silver, and K.~Ellis.
\newblock {EMPIRIC: Experiment-Driven Learning of Residual World Models for
  Robot Planning}.
\newblock \emph{arXiv preprint arXiv:2609.35047}, 2026.

\bibitem[Bar-Shalom and Tse(1976)]{barshalom1976caution}
Y.~Bar-Shalom and E.~Tse.
\newblock {Caution, Probing, and the Value of Information in the Control of
  Uncertain Systems}.
\newblock \emph{Annals of Economic and Social Measurement}, 5\penalty0
  (3):\penalty0 323--337, 1976.

\bibitem[Mesbah(2018)]{mesbah2018dual}
A.~Mesbah.
\newblock {Stochastic Model Predictive Control with Active Uncertainty
  Learning: A Survey on Dual Control}.
\newblock \emph{Annual Reviews in Control}, 45:\penalty0 107--117, 2018.

\bibitem[Lorenzen et~al.(2019)Lorenzen, Cannon, and
  Allg{\"o}wer]{lorenzen2019rmpc}
M.~Lorenzen, M.~Cannon, and F.~Allg{\"o}wer.
\newblock {Robust MPC with Recursive Model Update}.
\newblock \emph{Automatica}, 103:\penalty0 461--471, 2019.

\bibitem[Hu et~al.(2024)Hu, Isele, Bae, and Fisac]{hu2024dualshield}
H.~Hu, D.~Isele, S.~Bae, and J.~F. Fisac.
\newblock {Active Uncertainty Reduction for Safe and Efficient Interaction
  Planning: A Shielding-Aware Dual Control Approach}.
\newblock \emph{The International Journal of Robotics Research}, 43\penalty0
  (9):\penalty0 1382--1408, 2024.

\bibitem[Sadigh et~al.(2016)Sadigh, Sastry, Seshia, and Dragan]{sadigh2016info}
D.~Sadigh, S.~S. Sastry, S.~A. Seshia, and A.~D. Dragan.
\newblock {Information Gathering Actions over Human Internal State}.
\newblock In \emph{IEEE/RSJ International Conference on Intelligent Robots and
  Systems (IROS)}, pages 66--73, 2016.

\bibitem[Memmel et~al.(2024)Memmel, Wagenmaker, Zhu, Yin, Fox, and
  Gupta]{memmel2024asid}
M.~Memmel, A.~Wagenmaker, C.~Zhu, P.~Yin, D.~Fox, and A.~Gupta.
\newblock {ASID: Active Exploration for System Identification in Robotic
  Manipulation}.
\newblock In \emph{International Conference on Learning Representations
  (ICLR)}, 2024.

\bibitem[Milanese and Vicino(1991)]{milanese1991setmembership}
M.~Milanese and A.~Vicino.
\newblock {Optimal Estimation Theory for Dynamic Systems with Set Membership
  Uncertainty: An Overview}.
\newblock \emph{Automatica}, 27\penalty0 (6):\penalty0 997--1009, 1991.

\bibitem[Raue et~al.(2009)Raue, Kreutz, Maiwald, Bachmann, Schilling,
  Klingm{\"u}ller, and Timmer]{raue2009identifiability}
A.~Raue, C.~Kreutz, T.~Maiwald, J.~Bachmann, M.~Schilling, U.~Klingm{\"u}ller,
  and J.~Timmer.
\newblock {Structural and Practical Identifiability Analysis of Partially
  Observed Dynamical Models by Exploiting the Profile Likelihood}.
\newblock \emph{Bioinformatics}, 25\penalty0 (15):\penalty0 1923--1929, 2009.

\bibitem[Mookerjee et~al.(1985)Mookerjee, Bar-Shalom, and
  Molusis]{mookerjee1985turnoff}
P.~Mookerjee, Y.~Bar-Shalom, and J.~Molusis.
\newblock {Dual Control and Prevention of the Turn-Off Phenomenon in a Class of
  MIMO Systems}.
\newblock In \emph{Proceedings of the 24th IEEE Conference on Decision and
  Control}, pages 1888--1893, 1985.
\newblock \doi{10.1109/CDC.1985.268909}.

\bibitem[Si et~al.(2026)Si, Bi, Yang, Ni, Huang, Wang, Jiang, Wang, Li, Feng,
  Dong, and Zhou]{si2026pyrua}
R.~Si, J.~Bi, S.~Yang, R.~Ni, W.~Huang, Q.~Wang, S.~Jiang, D.~Wang, X.~Li,
  H.~Feng, Z.~Dong, and D.~Zhou.
\newblock {Fewer Tokens, Better Action: GPT-6 Astra Robot Agents with 14\%
  Higher Success Rate but 65\% Fewer Tokens}.
\newblock \emph{arXiv preprint arXiv:2610.01939}, 2026.

\bibitem[Black et~al.(2025)Black, Galliker, and Levine]{black2025rtc}
K.~Black, M.~Y. Galliker, and S.~Levine.
\newblock {Real-Time Execution of Action Chunking Flow Policies}.
\newblock In \emph{Advances in Neural Information Processing Systems
  (NeurIPS)}, 2025.

\bibitem[Karaman et~al.(2011)Karaman, Walter, Perez, Frazzoli, and
  Teller]{karaman2011anytime}
S.~Karaman, M.~R. Walter, A.~Perez, E.~Frazzoli, and S.~Teller.
\newblock {Anytime Motion Planning using the {RRT*}}.
\newblock In \emph{IEEE International Conference on Robotics and Automation
  (ICRA)}, pages 1478--1483, 2011.

\bibitem[Nguyen et~al.(2026)Nguyen, Jin, Pokorny, and Elmroth]{nguyen2026spo}
C.~Nguyen, S.~Jin, F.~T. Pokorny, and E.~Elmroth.
\newblock {Speculative Policy Orchestration: A Latency-Resilient Framework for
  Cloud-Robotic Manipulation}.
\newblock \emph{arXiv preprint arXiv:2603.19418}, 2026.

\bibitem[Leviathan et~al.(2023)Leviathan, Kalman, and
  Matias]{leviathan2023specdecoding}
Y.~Leviathan, M.~Kalman, and Y.~Matias.
\newblock {Fast Inference from Transformers via Speculative Decoding}.
\newblock In \emph{International Conference on Machine Learning (ICML)}, 2023.

\end{thebibliography}
\endgroup

\clearpage
\appendix

\section{The Harness}
\label{app:harness}
The harness applies programmed rules, while the agents choose reconstructions, parameter estimates, trajectories and expectations.

\paragraph{Context and tools.}
Each agent keeps a long-running conversation and workspace. When its context reaches the configured limit, it resumes in a fresh conversation from its notes and a handoff summary, with its workspace intact.
Agent commands run in a sandbox without network or direct robot access, and structured outputs are validated against their schemas before use.
\sense{} writes measurement and reconstruction code and uses a physics-check tool. \plan{} simulates candidate trajectories, inspects twin states and rendered views, replays recorded controls, compares residuals and asks \sense{} questions without requesting a full reconstruction.
Physics rollouts run inside the sandbox, and the harness brokers rendering and the question channel.
Observations, reconstruction scripts and estimates pass from \sense{} to \plan{}, and planning statements and revisions pass in the opposite direction.

\paragraph{Verification of revisions and execution.}
The harness validates that each revision names supported fields and selects revisions on the latest record as specified in Sec.~\ref{sec:problem}.
Parameter revisions undergo the replay check, while state corrections of moved bodies and articulated joints are applied without replay, since they synchronize rather than explain.
When the full revision fails, the harness accepts the individually passing changes if they also pass jointly (Algorithm~\ref{alg:spa}), or keeps the twin unchanged, and in either case it proceeds to gate the action plan.
Every committed chunk is then simulated in the updated twin from the measured robot state and must satisfy the robot limits and the planner's expectations.
The gate applies to the committed chunk, and the uncommitted remainder of the trajectory is checked when it is committed.
During execution, the monitor compares measured end-effector positions and gripper closures with their predictions, records alerts, and stops motion at the monitor stop threshold (App.~\ref{app:constants}).

\paragraph{Verification before continuation.}
A future chunk listed by \plan{} is gated without a new planning call when no reconstruction is requested, the preceding chunk executed in full, end-effector positions, orientations and gripper closures agree with their predictions within the continuation tolerances, no monitor alert occurred, and all expectations evaluable from robot measurements held. These measurable expectations cover wrist coordinates and gripper closure.
In addition, the depth residual of every object with comparable entries must not grow by more than its tolerance.
The listed chunk must then pass the same gate in the advanced twin from the current measured state. If any of these checks fails, \plan{} is called with the new evidence and the reason for the call.

\paragraph{Verification of completion.}
A completion proposal that passes the gate and finishes execution triggers a fresh observation and a completion check by \sense{}. Only this confirmation assigns a successful status. If \sense{} objects, its reasons are returned to \plan{} and the episode continues.

\paragraph{Execution records.}
Each round records its observation, twin, physics-check report, planning decision, proposed and accepted revisions, gate decision, and any executed controls with their starting state and measured response, together with the timing of every planning and perception call.
These records support online replay and offline refitting, and any recorded twin can be exported as a standalone MuJoCo scene.

\Needspace{24\baselineskip}
\section{Operational Tolerances and Budgets}
\label{app:constants}
\begin{table}[H]
\centering\footnotesize
\begin{tabular}{@{}llr@{}}
\toprule
Symbol & Meaning & Value \\
\midrule
 & max.\ end-effector step per action & 5\,cm, 0.35\,rad \\
 & end-effector deviation for monitor alert & 3\,cm \\
 & gripper closure deviation for monitor alert & 0.5 \\
 & end-effector deviation for monitor stop & 10\,cm \\
$\varepsilon^{\mathrm d}$ & depth tolerance at distance $z$ & $\max(0.02z,\kappa z^2)$ \\
$\kappa$ & depth noise coefficient, per observation & 0.014 to 0.10\,m$^{-1}$ \\
$\varepsilon^{\mathrm{ee}}$ & position tolerance from free-motion records & 6.6\,mm \\
$\varepsilon^{\mathrm g}$ & closure tolerance for revision comparison & 0.02 \\
 & reported closure resolution & 0.001 \\
$g_{\mathrm{held}}$ & closure deficit indicating held-open fingers & 0.02 \\
 & simulated displacement marking a body as moved & 1\,cm \\
 & orientation tolerance for continuation & 2$^\circ$ \\
 & candidate-rollout budget per round & 5 to 10 \\
 & question budget per round & 5 \\
$N_{\mathrm{sub}}$ & sub-agents per reconstruction & 4 \\
 & max.\ actions committed per round & 300 (${\approx}60$\,s) \\
$R$ & repairs after a refused plan & 2 \\
\bottomrule
\end{tabular}
\vspace{4pt}
\caption{Operational values and episode budgets. The rollout budget is fixed within an episode, and the depth coefficient is estimated per observation.}
\label{tab:constants}
\end{table}
Closure is normalized from 0 (open) to 1 (closed), and its comparison tolerance is larger than its reporting resolution to absorb replay and hardware discrepancies.
Position and closure deviations raise alerts, whereas only position deviation stops motion.

\section{Residuals and Conditional Consistency}
\label{app:error}

\paragraph{Motion and depth.}
For a record $\mathcal D$ with $H$ actions, let $p_{r,s}$ and $g_{r,s}$ be the measured end-effector position and normalized closure of arm $r$ after action $s$, $1\le s\le H$, and let $Z_n$ be the measured depth in camera $n$ after the chunk.
Hats denote predictions from the same recorded starting state and controls, dependence on $\mathcal D$ is suppressed, and norms are Euclidean.
The residuals are
\begin{align}
e^{\mathrm{ee}}_r&=\max_{s\le H}\lVert p_{r,s}-\hat p_{r,s}(\mathcal T)\rVert,
& e^{\mathrm g}_r&=\max_{s\le H}|g_{r,s}-\hat g_{r,s}(\mathcal T)|,\nonumber\\
e^{\mathrm d}_b&=\max_{n:\Omega_{b,n}(\mathcal T)\ne\varnothing}
\left|\operatorname*{median}_{\omega\in\Omega_{b,n}(\mathcal T)}
\big(Z_n(\omega)-\hat Z_n(\omega;\mathcal T)\big)\right|,
\label{eq:entries}
\end{align}
where $\Omega_{b,n}(\mathcal T)$ contains the pixels assigned to body $b$ in the rendered segmentation that also have valid measured depth.
A camera with an empty mask contributes no entry, and a body visible in no camera has no depth residual, so that missing observations neither support nor oppose a revision.
The depth statistic is the absolute value of a signed median. It is robust to outlier pixels and sensitive to a systematic offset along the viewing ray, which is the dominant error of a misplaced or mis-sized body. Lateral misalignment is captured through the change of the rendered mask, which depends on the candidate twin.

\paragraph{Tolerances.}
For depth at measured median distance $z$, the tolerance is $\max(0.02z,\kappa z^2)$.
The first term is the camera's stated relative accuracy, and the second is the stereo noise estimated from two depth frames $Z,Z'$ of the static scene. For each camera and observation, $\kappa$ is the 90th percentile over image cells of $|\operatorname{median}(Z-Z')|/(\sqrt2\,z^2)$.
Recorded coefficients range from 0.014 to 0.10\,m$^{-1}$, and the relative term is used alone when no estimate is available.
For revision acceptance, an object's tolerance is fixed from the camera with the largest baseline residual and its measured distance. At $z=0.5$\,m the relative term is 10\,mm, so depth residuals are compared at a resolution no finer than 10\,mm.

\paragraph{Resistance.}
A robot-only replay of the same controls gives free-motion positions $p^{\mathrm f}_{r,s}$ and closures $g^{\mathrm f}_{r,s}$.
Resistance is detected when $\lVert p_{r,s}-p^{\mathrm f}_{r,s}\rVert>\varepsilon^{\mathrm{ee}}$ or $g^{\mathrm f}_{r,s}-g_{r,s}>g_{\mathrm{held}}$, excluding deviations caused by different clamping under the per-action step limit.
Applying the same detector to measured and predicted motion gives first-resistance indices $s_r$ and $\hat s_r$, and the corresponding residual is
\begin{equation}
e^{\mathrm c}_r=\begin{cases}
\lVert p_{r,s_r}-\hat p_{r,\hat s_r}(\mathcal T)\rVert & \text{if both indices exist},\\
\text{overshoot along the free path} & \text{if only one exists},\\
0 & \text{otherwise},
\end{cases}
\label{eq:contact}
\end{equation}
where the overshoot is the largest distance by which the unresisted trajectory passes the resisting one along the free path during the resistance interval.
This entry uses the tolerance $\varepsilon^{\mathrm{ee}}$ and serves as a kinematic signature of contact onset.

\paragraph{Replay starting states.}
Replay starts from the snapshot stored with each executed chunk (Sec.~\ref{sec:problem}).
For comparisons on earlier records, the parameter revisions made since the snapshot are applied first, whereas corrections to bodies moved during that chunk and to articulated joints are omitted.
The latest replay endpoint is rendered against the latest observation.
This construction separates changes that explain a recorded transition from changes that synchronize the current state.

\paragraph{Finite coordinate scans.}
On request, the harness evaluates Eq.~\eqref{eq:open} on the endpoints of the candidate range, the perception estimate, the current estimate and the midpoint, with duplicates removed.
Motion residuals cover the chunks recorded since reconstruction, and depth is compared at the latest observation.
The reported interval spans the candidates that satisfy all tested comparisons and extends to the neighboring tested value on each side. If no candidate satisfies them, the full range is returned and marked as uninformative.

The scan is coordinate-wise, and coupled parameters show why this matters.
Let a measurement satisfy $y=\theta_1+\theta_2=1$, with candidate ranges $[0,1]$, nominal values $(0,0)$ and tolerance 0.1.
The set for $\theta_1$ with $\theta_2=0$ is $[0.9,1]$, and likewise for $\theta_2$, so their product predicts $y\ge1.8$ and excludes the consistent pair $(0.5,0.5)$.
\plan{} therefore treats $\widetilde{\mathcal W}_t$ as a source of candidate twins to test rather than as the joint set $\mathcal W_t$, and a measurement that depends on only one of the coupled parameters resolves such coupling.

\section{Proofs}
\label{app:dual}

\paragraph{Proof of Proposition~\ref{prop:rsi}.}
A chunk executes only when $\delta\neq\bot$ after the planning loop. The variable $\delta$ becomes a chunk only on the continuation line, which requires $\mathrm G(\hat{\mathcal T},\delta)=1$ and sets $\hat{\mathcal T}^+=\hat{\mathcal T}$, or inside the planning loop, where $\delta$ is reset to $\bot$ unless $\mathrm G(\hat{\mathcal T}^+,\delta)=1$. The executed chunk therefore passed $\mathrm G$ in $\hat{\mathcal T}^+$, the twin from which it executes.
\textsc{RSI} returns $\hat{\mathcal T}\oplus\Delta$ only if $A(\Delta)$, $\hat{\mathcal T}\oplus\Delta^{\mathrm{ind}}$ only if $A(\Delta^{\mathrm{ind}})$, and otherwise $\hat{\mathcal T}$, and $A$ evaluates Eq.~\eqref{eq:accept} on the replayable changes for the latest record $\bar{\mathcal D}$.
The only return of success follows $\textsc{Check}_\pi(\textsc{Observe}(),\ell)$, whose observation is taken after the last executed chunk. \qed

\paragraph{Proof of Lemma~\ref{lem:disc}.}
Let $y_i$ be the measured value. If both $\theta$ and $\theta'$ remained in $\mathcal W_{t+1}$, Eq.~\eqref{eq:belief} would give $d_i(y_i,\hat y_i(\theta;\mathcal D_t))\le\varepsilon_{i,t}$ and $d_i(y_i,\hat y_i(\theta';\mathcal D_t))\le\varepsilon_{i,t}$, and symmetry and the triangle inequality would give $d_i(\hat y_i(\theta;\mathcal D_t),\hat y_i(\theta';\mathcal D_t))\le2\varepsilon_{i,t}$, a contradiction. \qed

The pseudometric condition holds for the end-effector and closure entries of Eq.~\eqref{eq:entries}, which are maximum norms over the chunk. The depth entry, a signed median over a candidate-dependent mask, and the case-defined resistance entry of Eq.~\eqref{eq:contact} do not have this form. For them the lemma gives no guarantee, although they still constrain $\mathcal W_{t+1}$.

The lemma is worst-case, since it holds for every measured value whatever the noise. The factor $2$ is tight when each $d_i$ is induced by a norm. If the predictions differ by at most $2\varepsilon_{i,t}$ in every measured entry, the measurement midway between them in each entry keeps both hypotheses, so no exclusion is guaranteed. The lemma concerns exclusion from $\mathcal W_{t+1}$ and holds whatever the fidelity of the twin. Whether the excluded hypothesis is the wrong one depends on the error that no parameter value removes, which enters through $\gamma$ in App.~\ref{app:transfer}.
For the implemented scan, the same argument shows that, of two grid candidates that differ only in $\theta_k$, at least one fails the scan's test whenever their predictions differ by more than $2(\varepsilon_{i,t}+\bar e_{i})$, where $\bar e_i$ bounds the best grid candidate's residual. Because the reported interval is padded to the neighboring grid values, a failing candidate can still lie inside it.

\paragraph{Proof of Proposition~\ref{prop:probe}.}
Committing the sensitive chunk directly costs $T_P+T_A$ and, with probability $1-q$, an additional $C_{\mathrm{fail}}$, for an expected time $T_P+T_A+(1-q)C_{\mathrm{fail}}$.
Probing first costs $C_{\mathrm{probe}}$, after which the sensitive chunk succeeds by assumption, for a total of $C_{\mathrm{probe}}+T_P+T_A$. Both totals contain $T_P+T_A$, so the comparison reduces to $C_{\mathrm{probe}}$ against $(1-q)C_{\mathrm{fail}}$.
Probing is better exactly when $C_{\mathrm{probe}}<(1-q)C_{\mathrm{fail}}$, that is, when $q<1-C_{\mathrm{probe}}/C_{\mathrm{fail}}=q^\star$. Since $q^\star$ is affine in $C_{\mathrm{inc}}$ with slope $-1/C_{\mathrm{fail}}<0$, it is strictly decreasing. \qed

The comparison is deliberately minimal, and two extensions follow from the same computation. If the probe only raises the success probability from $q$ to $q'\in(q,1]$, probing pays exactly when $C_{\mathrm{probe}}<(q'-q)C_{\mathrm{fail}}$, that is, when $q<q'-C_{\mathrm{probe}}/C_{\mathrm{fail}}=q^\star-(1-q')$, so a partially informative probe lowers the threshold. If $L$ probes are needed, $C_{\mathrm{probe}}$ contains $L\,C_{\mathrm{inc}}$, and the threshold falls with slope $-L/C_{\mathrm{fail}}$, so an inexpensive incorporation matters $L$ times as much.
More generally, whenever the expected time of the probing branch increases with $C_{\mathrm{inc}}$ and that of the direct branch does not depend on it, the set of beliefs in which probing is preferred shrinks as $C_{\mathrm{inc}}$ grows. This monotonicity, rather than the particular threshold, is what Sec.~\ref{sec:theory} uses.

\section{A Sufficient Condition for Transfer}
\label{app:transfer}
The gate certifies a chunk in the nominal twin. The following condition states when the certified outcome also holds on the robot.
Let $\eta_m(\theta)$ be the predicted signed margin of task predicate $m$ along a fixed candidate trajectory, positive exactly when the predicate holds, and define the physical margin likewise.
With fixed positive scales $\sigma_m$ in the units of the margins, define dimensionless margins $f_m(\theta)=\eta_m(\theta)/\sigma_m$, and let $f_m^{\mathrm{real}}$ be the corresponding margin of the physical trajectory under the same controls.

\begin{proposition}[A sufficient condition for transfer]
\label{prop:transfer}
Let $B$ be a bounded box of parameter vectors and $\gamma\ge0$. Suppose that some $\theta^\star\in B$ satisfies $|f_m(\theta^\star)-f_m^{\mathrm{real}}|\le\gamma$ for every predicate, and that $B\cap\{\theta:\min_m f_m(\theta)>\gamma\}$ is convex.
If $\min_m f_m(\theta)>\gamma$ at every corner of $B$, all these predicates hold on the physical trajectory.
\end{proposition}
\begin{proof}
The convex set contains every corner of $B$ and hence their convex hull, which for a bounded box is $B$ itself. It therefore contains $\theta^\star$.
Consequently $f_m(\theta^\star)>\gamma$ for every $m$, and the discrepancy bound gives $f_m^{\mathrm{real}}>0$.
\end{proof}
The proposition separates the two sources of transfer error. The box $B$ accounts for parameter uncertainty, for example a box containing $\mathcal W_t$ when $\mathcal W_t$ contains $\theta^\star$, and $\gamma$ accounts for the error that remains at $\theta^\star$, such as unmodeled dynamics, which no parameter revision removes.
Checking the corners of $B$ is an optional robustness test available to \plan{}. Convexity is an assumption of the proposition rather than a property guaranteed by the simulator, and predicates without a signed margin fall outside the argument.

\section{Reducing Latency}
\label{sec:spec}
Since model inference dominates the critical path of Eq.~\eqref{eq:cost}, latency is reduced by issuing fewer and shorter requests~\citep{si2026pyrua,galbot2026astra} and by overlapping the work of the two modules.
Local revision avoids a reconstruction whenever a state or parameter change explains the evidence, and the two extensions below parallelize the remaining work.

\paragraph{Parallel perception.}
\sense{} can divide a reconstruction among up to $N_{\mathrm{sub}}$ sub-agents, for example by camera, region or object.
The sub-agents receive the observation and the lead's workspace and return partial twins, which the lead merges, recording disagreements as uncertain quantities.
With sufficient compute, the perception critical path becomes the longest concurrent branch plus the merge.

\paragraph{Planning ahead of perception.}
Related systems overlap inference with motion or speculate on future computation~\citep{black2025rtc,karaman2011anytime,nguyen2026spo,leviathan2023specdecoding}.
In the same spirit, \plan{} can begin from a draft or predicted twin while \sense{} reconstructs, provided that the gate uses the latest completed twin, including accepted revisions, rather than the draft.

\begin{proposition}[Gating a speculative proposal]
\label{prop:spec}
If every speculative proposal is rechecked in the latest completed twin at the state used for execution, an executed chunk satisfies the same nominal gate conditions as a sequentially planned chunk.
\end{proposition}
This follows from applying the same gate after reconstruction, and holding the robot still during reconstruction aligns its measured starting state with the one assumed by the plan.
With sensing, planning and motion durations $T_S$, $T_P$ and $T_A$, and the same planning work in both schedules, overlapping perception and planning costs $\max(T_S,T_P)+T_A+T_R$, where $T_R$ is any additional repair time, compared with $T_S+T_P+T_A$ sequentially. Since $T_S+T_P-\max(T_S,T_P)=\min(T_S,T_P)$, the overlap is faster exactly when $T_R<\min(T_S,T_P)$.

\end{document}